%% file: BPAI-17-DataWFNets.tex
\relax
\documentclass{llncs}

\usepackage{url}
\usepackage{subcaption}
\usepackage{xspace}
\usepackage{amssymb,amsmath}
\usepackage{color, colortbl}
\usepackage{booktabs}
\usepackage{multirow}
\usepackage[obeyFinal]{todonotes}
\usepackage[defblank]{paralist}
\usepackage{wrapfig}
\usepackage{totcount}                
\regtotcounter{@todonotes@numberoftodonotes}
\usepackage{newtxtext}
\usepackage{tikz}
\usetikzlibrary{petri}
\usetikzlibrary{matrix}
\usepackage{eurosym}

\usepackage[sharp]{easylist}

\usepackage{cleveref}
\crefname{section}{\textsection}{\textsection}
\Crefname{section}{\textsection}{\textsection}

\input{macros}

\newcommand{\onlytechrep}[1]{#1}
\newcommand{\onlypaper}[1]{}

\onlypaper{

}

\begin{document}

\title{Enhancing workflow-nets with data for trace completion}

\author{Riccardo De Masellis\inst{1}\and Chiara Di Francescomarino\inst{1} \and \\ Chiara Ghidini\inst{1} \and Sergio Tessaris\inst{2}}

\institute{FBK-IRST, Italy
  \email{\{r.demasellis,dfmchiara,ghidini\}@fbk.eu}
\and
  Free University of Bozen-Bolzano, Italy
  \email{tessaris@inf.unibz.it}
}

\maketitle

\setlength{\intextsep}{5pt}%




\begin{abstract}
  \input{0-abstract}
\end{abstract}


\input{1-introduction}
\input{2-preliminaries}
\input{3-framework}
\input{4-encoding-traces}
\input{5-encoding}

\input{6-relatedworks}

\input{7-conclusions}

\onlytechrep{
\appendix
\input{additional-backgrounds}
\input{additional-framework}
\input{additional-encoding-traces}

\input{additional-encoding}

}

\bibliographystyle{splncs03}
\bibliography{BPAI-17-DataWFNets-bib}

\end{document}

%% file: macros.tex
\usepackage{xspace}
\usepackage{ifthen}

\usepackage{listings}
\newcommand{\klng}{\ensuremath{\mathcal{K}}\xspace}
\newcommand{\Declare}{{\sc Declare}\xspace}

\newcommand{\pres}[1]{\ensuremath{{}^{\bullet}#1}\xspace}
\newcommand{\posts}[1]{\ensuremath{#1^{\bullet}}\xspace}

\newcommand{\pdom}[1]{\ensuremath{\Omega{#1}}}

\newcommand{\tuple}[1]{\ensuremath{\langle{#1}\rangle}}

\newcommand{\plantopn}[1]{\ensuremath{\Phi(#1)}\xspace}

\newcommand{\fire}[1]{\ensuremath{\stackrel{{#1}}{\rightarrow}}}

\newcommand{\tstep}[2]{\ensuremath{{#1}^{#2}}\xspace}

\newcommand{\ournet}{DAW-net\xspace}

\newcommand{\pidRel}[2][]{\ensuremath{#2^{\wp\ifthenelse{\equal{#1}{}}{}{(#1)}}}\xspace}

\newcommand{\C}{\mathcal{C}} \newcommand{\D}{\mathcal{D}}
 
 \renewcommand{\H}{\mathcal{H}}
 
 \renewcommand{\L}{\mathcal{L}}
\newcommand{\M}{\mathcal{M}}

\renewcommand{\S}{\mathcal{S}} 
 \newcommand{\V}{\mathcal{V}}

\newcommand{\set}[1]{\{ #1 \}}

\newcommand{\writef}{\ensuremath{\textsf{wr}}}

\newcommand{\guardf}{\ensuremath{\textsf{gd}}}
\newcommand{\varset}{\ensuremath{\mathcal{V}}}
\newcommand{\domf}{\ensuremath{\textsf{dm}}}
\newcommand{\guardlang}{\ensuremath{\mathcal{L}(\mathcal{D})}}

\newcommand{\ordf}{\ensuremath{\textsf{ord}}}
\newcommand{\img}{\ensuremath{img}}

\newcommand{\assign}{\eta}
\newcommand{\deff}{\textsf{def}}
\newcommand{\dmodel}{\mathcal{D}}
\newcommand{\nmodel}{\mathcal{N}}

\newcommand{\mrk}{M}

\newcommand{\wkfProj}{\ensuremath{\Pi}\xspace}

\newcommand{\kvar}[1]{\ensuremath{\text{var}_{#1}}\xspace}
\newcommand{\kvardom}[1]{\ensuremath{\text{dom}_{#1}}\xspace}
\newcommand{\kguard}[1]{\ensuremath{\text{grd}_{#1}}\xspace}
\newcommand{\kvardef}[1]{\ensuremath{\text{def}_{#1}}\xspace}
\newcommand{\kvarchange}[1]{\ensuremath{\text{chng}_{#1}}\xspace}
\newcommand{\katom}{\ensuremath{\xi}\xspace}
\newcommand{\kord}{\ensuremath{\text{ord}}\xspace}

\newcommand{\cvalue}[1]{\texttt{#1}}
\newcommand{\Nat}{\mathbb{N}}


%% file: 0-abstract.tex

The growing adoption of IT-systems for modeling and executing (business) processes or services has thrust the scientific investigation towards techniques and tools which support more complex forms of process analysis. Many of them, such as conformance checking, process alignment, mining and enhancement, rely on \emph{complete} observation of past (tracked and logged) executions. In many real cases, however, the lack of human or IT-support on all the steps of process execution, as well as information hiding and abstraction of model and data, result in incomplete log information of both data and activities. This paper tackles the issue of automatically repairing traces with missing information by notably considering not only activities but also data manipulated by them. Our technique recasts such a problem in a reachability problem and provides an encoding in an action language which allows to virtually use any state-of-the-art planning to return solutions.



%% file: 1-introduction.tex

\section{Introduction}

The use of IT systems for supporting business activities has brought to a large diffusion of \emph{process mining} techniques and tools that offer business analysts the possibility to observe the current process execution, identify deviations from the model, perform individual and aggregated analysis on current and past executions. 
\begin{wrapfigure}{r}{0.5\linewidth}
   \centering
     \includegraphics[width=.45\textwidth]{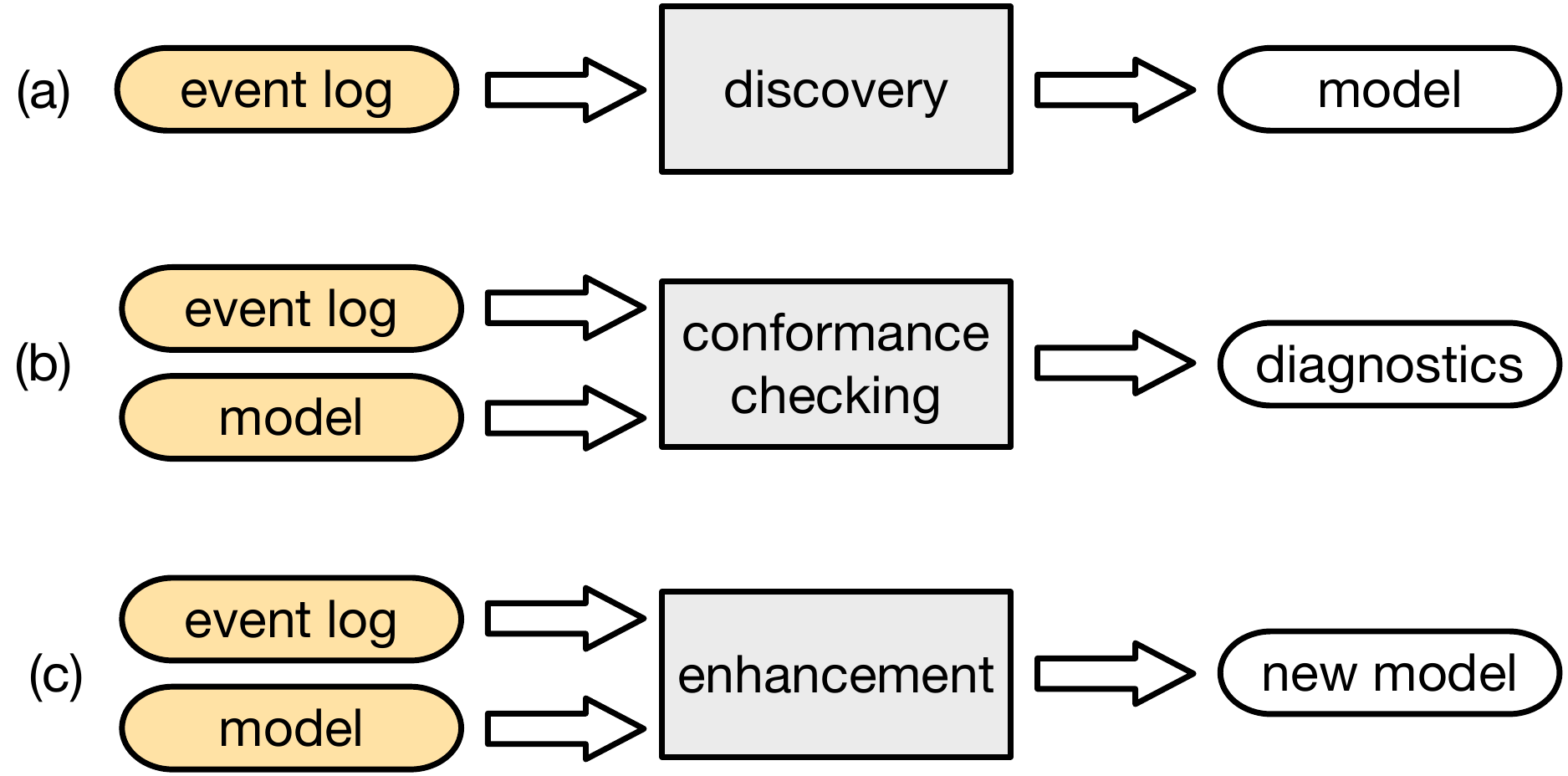}
   \caption{The three types of process mining.}
   \label{fig:imgs_ProcessMiningTypes}
 \end{wrapfigure}
According to the process mining manifesto, all these techniques and tools can be grouped in three basic types: process discovery, conformance checking and process enhancement (see Figure \ref{fig:imgs_ProcessMiningTypes}), and require in input an \emph{event log} and, for conformance checking and enhancement, a \emph{(process) model}.
A log, usually described in the IEEE standard XES format\footnote{\url{http://www.xes-standard.org/}}, is a set of execution traces (or cases) each of which is an ordered sequence of events carrying a payload as a set of attribute-value pairs.
Process models instead provide a description of the scenario at hand and can be constructed using one of the available Business Process Modeling Languages, such as BPMN, YAWL and \Declare.

Event logs are therefore a crucial ingredient to the accomplishment of process mining. Unfortunately, a number of difficulties may hamper the availability of event logs. Among these are partial event logs, where the execution traces may bring only \textbf{partial information} in terms of which process activities have been executed and what data or artefacts they produced.
Thus repairing incomplete execution traces by reconstructing the missing entries becomes an important task to enable process mining in full, as noted in recent works such as~\cite{Rogge-Solti2013,Di-Francescomarino-C.:2015aa}. While these works deserve a praise for having motivated the importance of trace repair and having provided some basic techniques for reconstructing missing entries using the knowledge captured in process models, they all focus on event logs (and process models) of limited expressiveness. In fact, they all provide techniques for the reconstruction of control flows, thus completely ignoring the data flow component.
This is a serious limitation, given the growing practical and theoretical efforts to extend business process languages with the capability to model complex data objects, along with the traditional control flow perspective~\cite{calvanese-13-PODS-keynote}. 
  
In this paper we show how to exploit state-of-the-art planning techniques to deal with the repair of data-aware event logs in the presence of imperative process models. Specifically we will focus on the well established Workflow Nets~\cite{aalst_soundness:2010}, a particular class of Petri nets that provides the formal foundations of several process models, of the YAWL language and have become one of the standard ways to model and analyze workflows.  
In particular we provide: 
\begin{compactenum}
	\item a modeling language \ournet, an extension of the workflow nets with data formalism introduced in \cite{sidorovastahletal:2011} so to be able to deal with even more expressive data (Section~\ref{sec:ourframework}); 
	\item a recast of data aware trace repair as a reachability problem in \ournet (Section~\ref{sec:encoding:traces});
	\item a sound and complete encoding of reachability in \ournet in a planning problem so to be able to deal with trace repair using planning (Section~\ref{sec:encoding:planning}). 
\end{compactenum}
The solution of the problem are all and only the repairs of the partial trace compliant with the \ournet model. The advantage of using automated planning techniques is that we can exploit the underlying logic language to ensure that generated plans conform to the observed traces without resorting to ad hoc algorithms for the specific repair problem.
The theoretical investigation presented in this work provides an important step forward towards the exploitation of mature planning techniques for the trace repair w.r.t. data-aware processes.

%% file: 2-preliminaries.tex
\section{Preliminaries}
\subsection{The Workflow Nets modeling language} 
\label{sub:subsection_name}

Petri Nets (PN) is a modeling language for the description of distributed systems that has widely been applied to the description and analysis of business processes~\cite{vanderaalst:1998}.
The classical PN is a directed bipartite graph with two node types, called \emph{places} and \emph{transitions}, connected via directed arcs. Connections between two nodes of the same type are not allowed.
\begin{definition}[Petri Net]
	\label{def:PN}
  A \emph{Petri Net} is a triple $\tuple{P,T,F}$ where $P$ is a set of \emph{places}; $T$ is a set of \emph{transitions}; $F\subseteq (P \times T) \cup (T \times P)$ is the flow relation describing the arcs between places and transitions (and between transitions and places).
\end{definition}
\begin{wrapfigure}{r}{0.35\linewidth}
  \centering
    \fbox{\includegraphics[width=.6\linewidth]{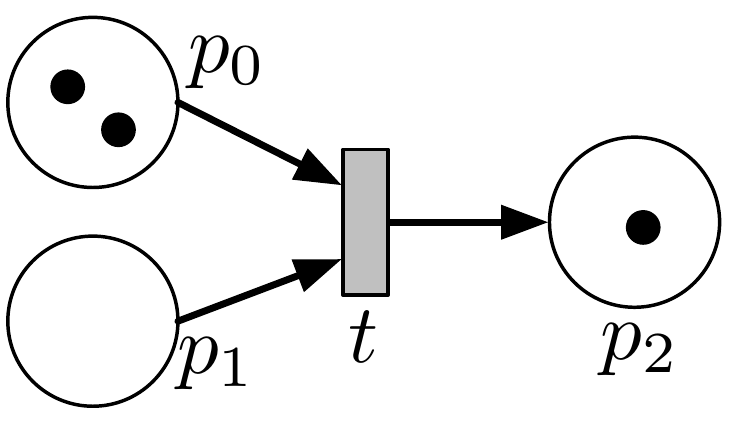}}
  \caption{A Petri Net.}
  \label{fig:imgs_PNsample}
\end{wrapfigure}
The \emph{preset} of a transition $t$ is the set of its input places: $\pres{t} = \{p \in P \mid (p,t) \in F\}$. The \emph{postset} of $t$ is the set of its output places: $\posts{t} = \{p \in P \mid (t,p) \in F\}$. Definitions of pre- and postsets of places are analogous.

Places in a PN may contain a discrete number of marks called tokens. Any distribution of tokens over the places, formally represented by a total mapping  $M : P\mapsto \mathbb{N}$, represents a configuration of the net called a \emph{marking}.  
PNs come with a graphical notation where places are represented by means of circles, transitions by means of rectangles and tokens by means of full dots within places. Figure~\ref{fig:imgs_PNsample} depicts a PN with a marking $M(p_0)=2$, $M(p_1)=0$, $M(p_2)=1$. The preset and postset of $t$ are $\{p_0, p_1\}$ and $\{p_2\}$, respectively.

\begin{figure*}[h]
  \centering
    \tikzstyle{place}=[circle,draw=black,thick]
    \tikzstyle{transition}=[text width=2cm,text centered,rectangle,draw=black!50,fill=black!20,thick]
    \scalebox{0.5}{\begin{tikzpicture}[font=\large]
      \node (start) [place,label=above:$start$] {};
      \node[transition,text width=2cm] (t1)  [right=.5cm of start, anchor=west] {T1:ask application documents};
      \node[place,label=right:$p_1$] (p1) [right=.8cm of t1, anchor=west] {};
      \node[transition] (t2)  [above right=1.5cm of p1, anchor=west] {T2:send student application};
      \node[transition] (t3)  [below right=1.5cm of p1, anchor=west] {T3:send worker application};
      \node[place,label=above:$p_2$] (p2) [right=.5cm of t2] {};
      \node[place,label=below:$p_3$] (p3) [right=.5cm of t3] {};
      \node[transition] (t4)  [right=.5cm of p2, anchor=west] {T4:fill student request};
      \node[transition] (t5)  [right=.5cm of p3, anchor=west] {T5:fill worker request};
      \node[place,label=left:$p_4$] (p4) [right=8cm of p1, anchor=west] {};
      \node[transition,text width=3cm] (t6)  [above right=2cm of p4] {T6:local credit officer approval};
      \node[transition,text width=3cm] (t7)  [right=1.3cm of p4] {T7:senior credit officer approval};
      \node[transition,text width=3cm] (t8)  [below right=2 cm of p4] {T8:bank credit committee approval};
      \node[place,label=above right:$p_5$] (p5) [right=5.5cm of p4, anchor=west] {};
      \node[transition,text width=.5cm] (t9)  [right=.5cm of p5] {T9};
      \node[place,label=above:$p_6$] (p6) [above=1cm of t9] {};
      \node[place,label=below:$p_7$] (p7) [below=1cm of t9] {};
      \node[transition] (t10)  [right=1cm of p6] {T10:send approval to customer};
      \node[transition] (t11)  [right=1cm of p7] {T11:store approval in branch};
      \node[place,label=above:$p_8$] (p8)  [right=1cm of t10] {};
      \node[place,label=below:$p_9$] (p9)  [right=1cm of t11] {};
      \node[transition,text width=1.5cm] (t12)  [below=.8cm of p8] {T12: issue loan};
      \node[place,label=above:$end$] (end)  [right=.5cm of t12] {};
      \draw [->,thick] (start.east) -- (t1.west);
      \draw [->,thick] (t1.east) -- (p1.west);
      \draw [->,thick] (p1.north) to [bend left=45] node[fill=white,inner sep=1pt, near start] {{loanType=s}} (t2.west);
      \draw [->,thick] (p1.south) to [bend right=45] node[fill=white,inner sep=1pt, near start, swap] {{loanType=w}} (t3.west);
      \draw [->,thick] (t2.east) to (p2.west);
      \draw [->,thick] (t3.east) to (p3.west);
      \draw [->,thick] (p2.east) to (t4.west);
      \draw [->,thick] (p3.east) to (t5.west);
      \draw [->,thick] (t4.east) to [bend left=45] (p4.north);
      \draw [->,thick] (t5.east) to [bend right=45] (p4.south);
      \draw [->,thick] (p4.north) to [bend left=45] node[fill=white,inner sep=1pt, midway] {{request$\leq 5k$}} (t6.west);
      \draw [->,thick] (p4.south) to [bend right=45] node[fill=white,inner sep=1pt, midway,swap] {{request$\geq 100k$}} (t8.west);
      \draw [->,thick] (p4.east) to node[fill=white,inner sep=1pt, midway,swap] {else} (t7.west);
      \draw [->,thick] (t6.east) to [bend left=45] (p5.north);
      \draw [->,thick] (t8.east) to [bend right=45] (p5.south);
      \draw [->,thick] (t7.east) to  (p5.west);
      \draw [->,thick] (p5.east) to (t9.west);
      \draw [->,thick] (t9.north) to (p6.south);
      \draw [->,thick] (t9.south) to (p7.north);
      \draw [->,thick] (p6.east) to (t10.west);
      \draw [->,thick] (p7.east) to (t11.west);
      \draw [->,thick] (t10.east) to (p8.west);
      \draw [->,thick] (t11.east) to (p9.west);
      \draw [->,thick] (p8.south) to  (t12.north);
      \draw [->,thick] (p9.north) to  (t12.south);
      \draw [->,thick] (t12.east) to  (end.west);
    \end{tikzpicture}}
  \caption{A process as a Petri Net.}
  \label{fig:imgs_SampleWFnet}
\end{figure*}
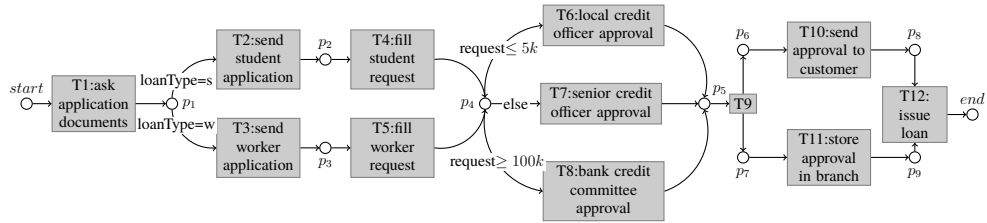
Process tasks are modeled in PNs as transitions while arcs and places constraint their ordering.
For instance, the process in Figure~\ref{fig:imgs_SampleWFnet}\footnote{For the sake of simplicity we only focus here on the, so-called, happy path, that is the successful granting of the loan.}
exemplifies how PNs can be used to model parallel and mutually exclusive choices, typical of business processes: sequences \emph{T2;T4}-\emph{T3;T5} and transitions \emph{T6-T7-T8} are indeed placed on mutually exclusive paths. Transitions \emph{T10} and \emph{T11} are instead placed on parallel paths. Finally, \emph{T9} is needed to prevent connections between nodes of the same type.

The expressivity of PNs exceeds, in the general case, what is needed to model business processes, which typically have a well-defined starting point and a well-defined ending point. This imposes syntactic restrictions on PNs, that result in the following definition of a workflow net (WF-net)~\cite{vanderaalst:1998}.
%
\begin{definition}[WF-net]
A PN $\tuple{P,T,F}$ is a WF-net if it has a single source place $start$, a single sink place $end$, and every place and every transition is on a path from start to end, i.e., for all $n\in P\cup T$, $(start,n)\in F^*$ and $(n,end)\in F^*$, where $F^*$ is the reflexive transitive closure of $F$.
\end{definition}

A marking in a WF-net represents the \emph{workflow state} of a single case. The semantics of a PN/WF-net, and in particular the notion of \emph{valid firing}, defines how transitions route tokens through the net so that they correspond to a process execution. 

\begin{definition}[Valid Firing]
	\label{def:Firing}
A firing of a transition $t\in T$ from $M$ to $M'$ is \emph{valid}, in symbols $M \fire{t} M'$, iff
\begin{enumerate}
  \item $t$ is enabled in $M$, i.e., $\{ p\in P\mid M(p)>0\}\supseteq \pres{t}$; and
  \item the marking $M'$ is such that for every $p\in P$:
  \begin{displaymath}
	  \small
    M'(p) =
    \begin{cases}
      M(p)-1 & \text{if $p\in \pres{t}\setminus\posts{t}$}\\
      M(p)+1  & \text{if $p\in \posts{t}\setminus\pres{t}$}\\
      M(p) & \text{otherwise}
    \end{cases}
  \end{displaymath}
\end{enumerate}
\end{definition}
Condition 1.~states that a transition is enabled if all its input places contain at least one token; 2.~states that when $t$ fires it consumes one token from each of its input places and produces one token in each of its output places. 

A \emph{case} of a WF-Net  is a sequence of valid firings 
$M_0 \fire{t_1} M_1, M_1 \fire{t_2} M_2, \ldots, M_{k-1} \fire{t_k} M_k$ 
where $M_0$ is the marking indicating that there is a single token in $start$.

\begin{definition}[$k$-safeness]
 A marking of a PN is $k$-safe if the number of tokens in all places is at most $k$. A PN is $k$-safe if the initial marking is $k$-safe and the marking of all cases is $k$-safe.
\end{definition}
From now on we concentrate on $1$-safe nets, which generalize the class of \emph{structured workflows} and are the basis for best practices in process modeling~\cite{kiepuszewskihofstedeetal:2013}. We also use safeness as a synonym of 1-safeness. 
It is important to notice that our approach can be seamlessly generalized to other classes of PNs, as long as it is guaranteed that they are $k$-safe. This reflects the fact that the process control-flow is well-defined (see~\cite{heesidorovaetal:2003}).




\vspace{5pt}
\noindent
{\textbf{Reachability on Petri Nets}.} The behavior of a PN  can be described as a transition system where states are markings and directed edges represent firings. Intuitively, there is an edge from $M_i$ to $M_{i+1}$ labeled by $t_i$ if $M_i \fire{t} M_{i+1}$
is a valid firing. Given a ``goal'' marking $M_g$, the reachability problem amounts to check if there is a path from the initial marking $M_0$ to $M_g$. Reachability on PNs (WF-nets) is of enormous importance in process verification as it allows for checking natural behavioral properties, such as satisfiability and soundness in a natural manner~\cite{Aalst97}.

\subsection{Trace repair} 
\label{sub:trace_repair}

One of the goals of process mining is to capture the  as-is  processes  as  accurately  as  possible: this is done by examining event logs that can be then exploited to perform 
the tasks in Figure~\ref{fig:imgs_ProcessMiningTypes}. In many cases, however, event logs are subject to data quality problems, resulting in \emph{incorrect} or \emph{missing} events in the log. In this paper we focus on the latter issue addressing the problem of \textbf{repairing execution traces that contain missing entries} (hereafter shortened in trace repair).  

The need for trace repair is motivated in depth in~\cite{Rogge-Solti2013}, where missing entities are described as a frequent cause of low data quality in event logs, especially when the definition of the business processes integrates activities that are not supported by IT systems due either to their nature (e.g. they consist of human interactions) or to the high level of abstraction of the description, detached from the implementation. A further cause of missing events are special activities (such as transition \emph{T9} in Figure~\ref{fig:imgs_SampleWFnet}) that are introduced in the model to guarantee properties concerning e.g., the structure of the workflow or syntactic constraints, but are never executed in practice. 

The starting point of trace repair are \emph{execution traces} and the knowledge captured in \emph{process models}. Consider for instance the model in Figure~\ref{fig:imgs_SampleWFnet} and the (partial) execution trace 
\{T3, T7\}.
By aligning the trace to the model using a replay-based approach or a planning based approach, the techniques presented in~\cite{Rogge-Solti2013} and~\cite{Di-Francescomarino-C.:2015aa} are able to exploit the events stored in the trace and the control flow specified in the model to reconstruct two possible repairs: 

\scalebox{0.9}{
\parbox{\linewidth}{%
\begin{align*}
\label{eq:trace-complete-nodata1}
\{T1, T3, T5, T7, T9, T10, T11, T12 \} \\
\{T1, T3, T5, T7, T9, T11, T10, T12 \}
\end{align*}
}}

Consider now a different scenario in which the partial trace reduces to \{$T7$\}. In this case, by using the control flow in Figure~\ref{fig:imgs_SampleWFnet} we are not able to reconstruct whether the loan is a student loan or a worker loan. This increases the number of possible repairs and therefore lowers the usefulness of trace repair. Assume nonetheless that the event log conforms to the XES standard and stores some observed data attached to $T7$ (enclosed in square brackets): 
\begin{equation*}
\label{eq:trace-data}
\{T7[request=\cvalue{60k},loan=\cvalue{50k}]\}
\end{equation*}
If the process model is able to specify how transitions can read and write variables, and furthermore some constraints on how they do it, the scenario changes completely. Indeed, assume that transition \emph{T4} is empowered with the ability to write the variable $request$ with a value smaller or equal than $\cvalue{30k}$ (being this the maximum amount of a student loan). Using this fact, and the fact that the request examined by \emph{T7} is greater than $\cvalue{30k}$, we can understand that the execution trace has chosen the path of the worker loan. Moreover, if the model specifies that variable \emph{loanType} is written during the execution of \emph{T1}, when the applicant chooses the type of loan she is interested to, we are able to infer that \emph{T1} sets variable \emph{loanType} to $\cvalue{w}$. 
This example, besides illustrating the idea of trace repair, also motivates why data are important to accomplish this task, and therefore why extending repair techniques beyond the mere control flow is a significant contribution to address data quality problems in event logs.  

\subsection{The planning language \klng} 
\label{sub:planning_language}
The main elements of action languages are \emph{fluents} and \emph{actions}. The former represent the state of the system which may change by means of actions. Causation statements describe the possible evolution of the states, and preconditions associated to actions describe which action can be executed according to the current state.
A planning problem in \klng~\cite{eiter_dlvk:2003} is specified using a Datalog-like language where fluents and actions are represented by literals (not necessarily ground). The specification includes the list of fluents, actions, initial state and goal conditions; also a set of statements specifies the dynamics of the planning domain using causation rules and executability conditions.
The semantics of \klng borrows heavily from Answer Set Programming (ASP) paradigm. In fact, the system enables the reasoning with partial knowledge and provides both weak and strong negation.

A \emph{causation rule} is a statement of the form
\begin{lstlisting}
    caused $f$ if $b_1$,$\ldots$, $b_k$, not $b_{k+1}$, $\ldots$, not $b_\ell$ 
               after $a_1$,$\ldots$, $a_m$, not $a_{m+1}$, $\ldots$, not $a_n$.
\end{lstlisting}
 The rule states that $f$ is true in the new state reached by executing (simultaneously) some actions, provided that $a_1,\ldots,a_m$ are known to hold while $a_{m+1},\ldots,a_n$ are not known to hold in the previous state (some of the $a_j$ might be actions executed on it), and $b_1,\ldots,b_k$ are known to hold while $b_{k+1},\ldots,b_\ell$ are not known to hold in the new state.
 Rules without the \lstinline|after| part are called \emph{static}.

An \emph{executability condition} is a statement of the form
\begin{lstlisting}
    executable $a$ if $b_1$,$\ldots$, $b_k$, not $b_{k+1}$, $\ldots$, not $b_\ell$.
\end{lstlisting}
 Informally, such a condition says that the action $a$ is eligible for execution in a state, if $b_1,\ldots,b_k$ are known to hold while $b_{k+1},\ldots,b_\ell$ are not known to hold in that state.

Terms in both kind of statements could include variables (starting with capital letter) and the statements must be safe in the usual Datalog meaning w.r.t.\ the first fluent or action of the statements.

%

A \emph{planning domain PD} is a tuple $\tuple{D,R}$ where $D$ is a finite set of action and fluent declarations and $R$ a finite set of  rules, initial state constraints, and executability conditions.

The semantics of the language is provided in terms of a transition system where the states are ASP models (sets of atoms) and actions transform the state according to the rules. A state transition is a tuple $t = \tuple{s, A, s'}$ where $s, s'$ are states and $A$ is a set of action instances. The transition is said to be legal if the actions are executable in the first state and both states are the minimal ones that satisfy all causation rules. Semantics of plans including default negation is defined by means of a Gelfond-Lifschitz type reduction to a positive planning domain.
A sequence of state transitions $\tuple{s_0, A_1, s_1}, \ldots,\tuple{s_{n-1}, A_n, s_n}$, $n \geq 0$, is a trajectory for PD, if $s_0$ is a legal initial state of PD and all $\tuple{s_{i-1}, A_i, s_i}$, are legal state transitions of PD.

A \emph{planning problem} is a pair of planning domain PD and a ground goal
\newline
\lstinline|$g_1,\ldots, g_m$, not $g_{m+1}$, $\ldots$, not $g_n$|
that is required to be satisfied at the end of the execution.


%% file: 3-framework.tex

\section{Framework}
\label{sec:ourframework}
In this section we suitably extend WF-nets to represent data and their evolution as transitions are performed.
In order for such an extension to be meaningful, i.e., allowing reasoning on data, it has to provide: (i) a model for representing data; (ii) a way to make decisions on actual data values; and (iii) a mechanism to express modifications to data.
Therefore, we enhance WF-nets with the following elements:
\begin{compactitem}
\item a set of variables taking values from possibly different domains (addressing (i));
\item queries on such variables used as transitions preconditions (addressing (ii))
\item variables updates and deletion in the specification of net transitions (addressing (iii)).
\end{compactitem}
Our framework follows the approach of state-of-the-art WF-nets with data~\cite{sidorovastahletal:2011,de_leoni:2013}, from which it borrows the above concepts, extending them by allowing reasoning on actual data values as better explained in Section~\ref{sec:related_works}. 

Throughout the section we use the WF-net in Figure~\ref{fig:imgs_SampleWFnet} extended with data as a running example.

\subsection{Data Model}\label{sec:datamodel}
As our focus is on trace repair, we follow the data model of the IEEE XES standard for describing logs,
 which represents data as a set of variables. Variables take values from specific sets on which a partial order can be defined. As customary, we distinguish between the data model, namely the intensional level, from a specific instance of data, i.e., the extensional level.

\begin{definition}[Data model]
A \emph{data model} is a tuple $\D = (\V, \Delta, \domf, \ordf)$ where:
\begin{compactitem}
\item $\V$ is a possibly infinite set of variables;
\item $\Delta = \set{\Delta_1, \Delta_2, \ldots}$ is a possibly infinite set of domains (not necessarily disjoint);
\item $\domf: \V \rightarrow \Delta$ is a total and surjective function which associates to each variable $v$ its domain $\Delta_i$;
\item $\ordf$ is a partial function that, given a domain $\Delta_i$, if $\ordf(\Delta_i)$ is defined, then it returns a \emph{partial order} (reflexive, antisymmetric and transitive) $\leq_{\Delta_i} \subseteq \Delta_i \times \Delta_i$.
\end{compactitem}
\end{definition}

A data model for the loan example is $\V = \{loanType$, $request, loan\}$, $\domf(loanType)=\set{\cvalue{w}, \cvalue{s}}$, $\domf(request) = \Nat$, $\domf(loan) = \Nat$, with $\domf(loan)$ and $\domf(loanType)$ being total ordered by the natural ordering $\leq$ 
 in $\Nat$.

An actual instance of a data model is simply a partial function associating values to variables.
\begin{definition}[Assignment]
Let $\D = \tuple{\V, \Delta, \domf, \ordf}$ be a data model. An \emph{assignment} for variables in $\V$ is a \emph{partial} function $\assign: \V \rightarrow \bigcup_i\Delta_i$ such that for each $v \in \V$, if $\assign(v)$ is defined, i.e., $v \in \img(\eta)$ where $\img$ is the image of $\eta$, then we have $\assign(v) \in \domf(v)$.
\end{definition}

We now define our boolean query language, which notably allows for equality and comparison. As will become clearer in Section~\ref{sec:our-net}, queries are used as \emph{guards}, i.e., preconditions for the execution of transitions.

\begin{definition}[Query language - syntax]
Given a data model, the language $\guardlang$ is the set of formulas $\Phi$ inductively defined according to the following grammar:
$$\Phi \quad := \quad true \mid \deff(v) \mid t_1 = t_2 \mid t_1 \leq t_2 \mid \neg \Phi_1 \mid \Phi_1 \land \Phi_2$$
where $v \in \V$ and $t_1, t_2 \in \V \cup \bigcup_i \Delta_i$.
\end{definition}

Examples of queries of the loan scenarios are $request \leq \cvalue{5k}$ or $loanType=\cvalue{w}$.
Given a formula $\Phi$ and an assignment $\assign$, we write $\Phi[\assign]$ for the formula $\Phi$ where each occurrence of variable $v \in \img(\assign)$ is replaced by $\assign(v)$.

\begin{definition}[Query language - semantics]
Given a data model $\dmodel$, an assignment $\assign$ and a query $\Phi \in \guardlang$ we say that $\dmodel, \assign$ satisfies $\Phi$, written $D, \assign \models \Phi$ inductively on the structure of $\Phi$ as follows:
\begin{compactitem}
\item $\dmodel, \assign \models true$;
\item $\dmodel, \assign \models \deff(v)$ iff $v \in \img(\assign)$;
\item $\dmodel, \assign \models t_1 = t_2$ iff $t_1[\assign], t_2[\assign] \not \in \V$ and $t_1[\assign] \equiv t_2[\assign]$;
\item $\dmodel, \assign \models t_1 \leq t_2$ iff $t_1[\assign], t_2[\assign] \in \Delta_i$ for some $i$ and $\ordf(\Delta_i)$ is defined and $t_1[\assign] \leq_{\Delta_i} t_2[\assign]$;
\item $\dmodel, \assign \models \neg \Phi$ iff it is not the case that $\dmodel, \assign \models \Phi$;
\item $\dmodel, \assign \models \Phi_1 \land \Phi_2$ iff $\dmodel, \assign \models \Phi_1$ and $\dmodel, \assign \models \Phi_2$.
\end{compactitem} 
\end{definition}

Intuitively, $\deff$ can be used to check if a variable has an associated value or not (recall that assignment $\eta$ is a partial function); equality has the intended meaning and $t_1 \leq t_2$ evaluates to true iff $t_1$ and $t_2$ are values belonging to the same domain $\Delta_i$, such a domain is ordered by a partial order $\leq_{\Delta_i}$ and $t_1$ is actually less or equal than $t_2$ according to $\leq_{\Delta_i}$.

\subsection{Data-aware net}\label{sec:our-net}

We now combine the data model with a WF-net and formally define how transitions are guarded by queries and how they update/delete data. The result is a Data-AWare net (\ournet) that incorporates aspects (i)--(iii) described at the beginning of Section~\ref{sec:ourframework}.

\begin{definition}[\ournet]\label{def:our-net}
  A \emph{\ournet} is a tuple $\langle$$\dmodel,$ $\nmodel,$  $\writef,$~$\guardf$$\rangle$ where:
\begin{compactitem}
    \item $\nmodel=\tuple{P,T,F}$ is a WF-net;
    \item $\dmodel=\tuple{\V, \Delta, \domf, \ordf}$ is a data model;
    \item $\writef: T \mapsto (\V' \mapsto 2^{\domf(\varset)})$, where $\varset'\subseteq\varset$, $\domf(\varset)=\bigcup_{v \in V} \domf(v)$ and $\writef(t)(v)\subseteq\domf(v)$ for each $v\in \varset'$, is a function that associates each transition to a (\emph{partial}) function mapping variables to a finite subset of their domain.
    \item $\guardf: T \mapsto \guardlang$ is a function that associates a guard to each transition.
  \end{compactitem}
\end{definition}

Function $\guardf$ associates a guard, namely a query, to each transition. The intuitive semantics is that a transition $t$ can fire if its guard $\guardf(t)$ evaluates to true (given the current assignment of values to data). Examples are $\guardf(T6)=request \leq \cvalue{5k}$ and $\guardf(T8)= \neg(request \leq \cvalue{99999})$.
Function $\writef$ is instead used to express how a transition $t$ modifies data: after the firing of $t$, each variable $v \in \V'$ can take any value among a specific finite subset of $\domf(v)$. We have three different cases:
\begin{compactitem}
 \item $\emptyset\subset \writef(t)(v) \subseteq \domf(v)$: $t$ nondeterministically assigns a value from $\writef(t)(v)$ to $v$;
  \item $\writef(t)(v) = \emptyset$: $t$ deletes the value of $v$ (hence making $v$ undefined);
  \item $v \not\in dom(\writef(t))$: value of $v$ is not modified by $t$.
\end{compactitem}
Notice that by allowing $\writef(t)(v) \subseteq \domf(v)$ in the first bullet above we enable the specification of restrictions for specific tasks. E.g., $\writef(T4):\set{request} \mapsto \set{\cvalue{0} \ldots \cvalue{30k}}$ says that $T4$ writes the $request$ variable and intuitively that students can request a maximum loan of $\cvalue{30k}$, while $\writef(T5):\set{request} \mapsto \set{\cvalue{0} \ldots \cvalue{500k}}$ says that workers can request up to $\cvalue{500k}$.

The intuitive semantics of $\guardf$ and $\writef$ is formalized next. We start from the definition of \ournet state, which includes both the state of the WF-net, namely its marking, and the state of data, namely the assignment. We then extend the notions of state transition and valid firing. 

\begin{definition}[\ournet state]
A \emph{state} of a \ournet $\tuple{\dmodel, \nmodel, \writef, \guardf}$ is a pair $(\mrk,\assign)$ where $\mrk$ is a marking for $\tuple{P,T,F}$ and $\assign$ is an assignment for $\dmodel$.
\end{definition}

\begin{definition}[\ournet Valid Firing]
Given a \ournet $\tuple{\dmodel, \nmodel, \writef, \guardf}$, a firing of a transition $t\in T$ is a \emph{valid firing} from $(\mrk,\assign)$ to $(\mrk',\assign')$, written as $(\mrk,\assign)\fire{t}(\mrk',\assign')$, iff conditions 1.\ and 2.\ of Def.~\ref{def:Firing} holds for $\mrk$ and $\mrk'$, i.e., it is a WF-Net valid firing, and
  \begin{compactenum}
  \item $\dmodel, \assign \models \guardf(t)$, 
  \item assignment $\assign'$ is such that, if $\textsc{wr} = \set{v \mid \writef(t)(v)\neq\emptyset}$, $\textsc{del}  = \set{ v \mid \writef(t)(v)=\emptyset}$:
  \begin{compactitem}
  	\item its domain $dom(\assign') = dom(\assign)\cup \textsc{wr} \setminus \textsc{del}$;
        \item for each $v\in dom(\assign')$:
        \begin{displaymath}
          \assign'(v) =
          \begin{cases}
            d \in \writef(t)(v) & \text{if $v\in \textsc{wr}$}\\
            \assign(v)  & \text{otherwise.}
          \end{cases}
        \end{displaymath}
  \end{compactitem}
\end{compactenum}
\end{definition}

Condition 1.~and 2.\ extend the notion of valid firing of WF-nets imposing additional pre- and postconditions on data, i.e., preconditions on $\assign$ and postconditions on $\assign'$. Specifically, 1.\ says that for a transition $t$ to be fired its guard $\guardf(t)$ must be satisfied by the current assignment $\assign$. Condition 2.\ constrains the new state of data: the domain of $\assign'$ is defined as the union of the domain of $\eta$ with variables that are written ($\textsc{wr}$), minus the set of variables that must be deleted ($\textsc{del}$). Variables in $dom(\assign')$ can indeed be grouped in three sets depending on the effects of $t$: (i) $\textsc{old} = dom(\eta) \setminus \textsc{wr}$: variables whose value is unchanged after $t$; (ii) $\textsc{new} = \textsc{wr} \setminus dom(\eta)$: variables that were undefined but have a value after $t$; and (iii) $\textsc{overwr} = \textsc{wr} \cap dom(\eta)$: variables that did have a value and are updated with a new one after $t$.
The final part of condition 2.\ says that each variable in $\textsc{new} \cup \textsc{overwr}$ takes a value in $\writef(t)(v)$, while variables in $\textsc{old}$ maintain the old value $\eta(v)$.

A \emph{case} of a \ournet is defined as 
 a case of a WF-net, with the only difference that the assignment $\assign_0$ of the initial state $(M_0, \assign_0)$ is empty, i.e., $dom(\assign_0)=\emptyset$.


%% file: 4-encoding-traces.tex

\section{Trace repair as reachability}\label{sec:encoding:traces}

In this section we provide the intuition behind our technique for solving the trace repair problem via reachability. Full details and proofs are contained in\onlypaper{~\cite{additional:ICAPS:2017}}\onlytechrep{ Appendices~\ref{app:preliminaries}--\ref{app:encoding}}. 

A \emph{trace} is a sequence of observed \emph{events}, each with a payload including the transition it refers to and its effects on the data, i.e., the variables updated by its execution.
%
Intuitively, a \ournet case is \emph{compliant} w.r.t.\ a trace if it contains all the occurrences of the transitions observed in the trace 
(with the corresponding variable updates) in the right order.

As a first step, we assume without loss of generality that \ournet models start with a special transition $start_t$ and terminate with a special transition $end_t$. 
Every process can be reduced to such a structure as informally illustrated in the left hand side of Figure~\ref{fig:trace-compilation} by arrows labeled with (1). Note that this change would not modify the behavior of the net: any sequence of firing valid for the original net can be extended by the firing of the additional transitions and vice versa.

\begin{figure*}[t]
 \begin{center}
  \includegraphics[width=.9\linewidth]{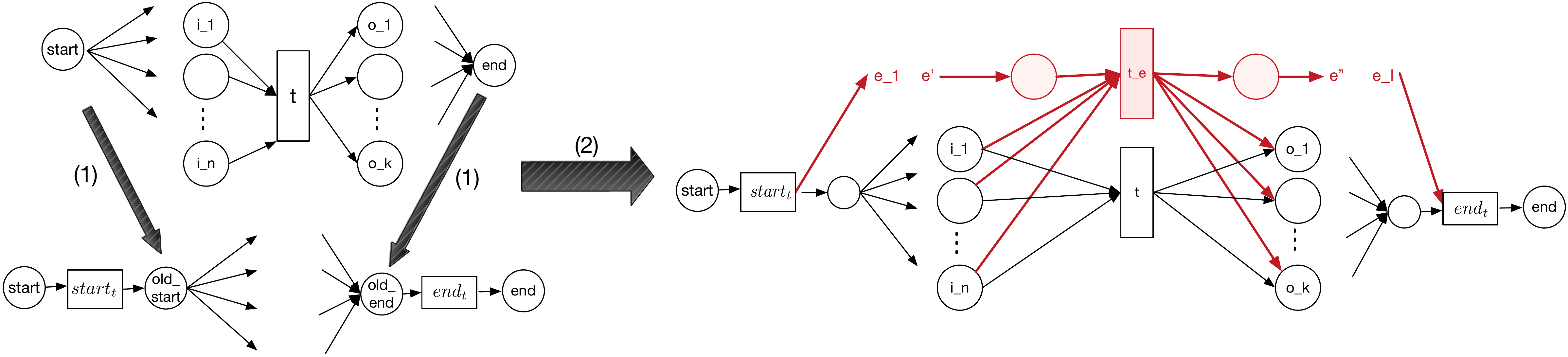}
\end{center}
\vspace{-.5cm}
 \caption{Outline of the trace ``injection''}
 \vspace{-.5cm}
 \label{fig:trace-compilation}
\end{figure*}

Next, we illustrate the main idea behind our approach by means of the right hand side of Figure~\ref{fig:trace-compilation}: we consider the observed events as transitions (in red) and we suitably ``inject'' them in the original \ournet. By doing so, we obtain a new model where, intuitively, tokens are forced to activate the red transitions of \ournet, when events are observed in the trace. When, instead,  there is no red counterpart, i.e., there is missing information in the trace, the tokens move in the black part of the model.
The objective is then to perform reachability for the final marking (i.e., to have one token in the $end$ place and all other places empty) over such a new model in order to obtain all and only the possible repairs for the partial trace.

More precisely, for each event $e$ with a payload including transition $t$ and some effect on variables we introduce a new transition $t_e$ in the model such that:
\begin{compactitem}
  \item $t_e$ is placed in parallel with the original transition $t$;
  \item $t_e$ includes an additional input place connected to the preceding event and an additional output place which connects it to the next event;
  \item $\guardf(t_e) = \guardf(t)$ and
  \item $\writef(t_e)$ specifies exactly the variables and the corresponding values updated by the event, i.e.\ if the event set the value of $v$ to $d$, then $\writef(t_e)(v) = \set{d}$; if the event deletes the variable $v$, then $\writef(t_e)(v) = \emptyset$.
\end{compactitem}


Given a trace $\tau$ and a \ournet $W$, it is easy to see that the resulting \emph{trace workflow} (indicated as $W^\tau$) is a strict extension of $W$ (only new nodes are introduced) and, since all newly introduced nodes are in a path connecting the start and sink places, it is a \ournet, whenever the original one is a \ournet net.

We now prove the soundness and completeness of the approach by showing that: \begin{inparaenum}[(1)] \item all cases of  $W^\tau$ are compliant with $\tau$; \item each case of $W^\tau$ is also a case of $W$ and \item if there is a case of $W$ compliant with $\tau$, then that is also a case for $W^\tau$\end{inparaenum}.

Property (1) is ensured by construction. For (2) and (3) we need to relate cases from $W^\tau$ to the original \ournet $W$. We indeed introduce a projection function $\wkfProj_\tau$ that maps elements from cases of the enriched \ournet to cases of elements from the original \ournet. 
Essentially, $\wkfProj_\tau$ maps newly introduced transitions $t_e$ to the corresponding transitions in event $e$, i.e., $t$, and also projects away the new places in the markings. 
Given that the structure of $W^\tau$ is essentially the same as that of $W$ with additional copies of transitions that are already in $W$, it is not surprising that any case for $W^\tau$ can be replayed on $W$ by mapping the new transitions $t_e$ into the original ones $t$, as shown by the following:
\begin{lemma}
  If $C$ is a case of $W^\tau$ then $\wkfProj_\tau(C)$ is a case of $W$.
\end{lemma}

This lemma proves that whenever we find a case on $W^\tau$, then it is an example of a case on $W$ that is compliant with $\tau$, i.e., (2). However, to reduce the original problem to reachability on \ournet, we need to prove that \emph{all} the $W$ cases compliant with $\tau$ can be replayed on $W^\tau$, that is, (3).
In order to do that, we can build a case for $W^\tau$ starting from the compliant case for $W$, by substituting the occurrences of firings corresponding to events in $\tau$ with the newly introduced transitions. 
The above results pave the way to the following: 

\begin{theorem}
	\label{Th:traceworkflow}
  Let $W$ be a \ournet and $\tau = (e_1,\ldots, e_n)$ a trace; then $W^\tau$ characterises all and only the cases of $W$ compatible with $\tau$. That is
  \begin{compactitem}
    \item[$\Rightarrow$] if $C$ is a case of $W^\tau$ containing $t_{e_n}$ then $\wkfProj_\tau(C)$ is compatible with $\tau$; and 
    \item[$\Leftarrow$] if $C$ is a case of $W$ compatible with $\tau$, then there is a case $C'$ of $W^\tau$ s.t.\ $\wkfProj_\tau(C') = C$.
  \end{compactitem}
\end{theorem}
Theorem~\ref{Th:traceworkflow} provides the main result of this section and is the basis for the reduction of the trace repair for $W$ and $\tau$ to the reachability problem for $W^\tau$. In fact, by enumerating all the cases of $W^\tau$ reaching the final marking (i.e.\ a token in $end$) we can provide all possible repairs for the partial observed trace.
Moreover, the transformation generating $W^\tau$ is preserving the safeness properties of the original workflow:
\begin{lemma}
  Let $W$ be a \ournet and $\tau$ a trace of $W$. If $W$ is $k$-safe then $W^\tau$ is $k$-safe as well.
\end{lemma}
This is essential to guarantee the decidability of the reasoning techniques described in the next section.


%% file: 5-encoding.tex

\section{Reachability as a planning problem}
\label{sec:encoding:planning}
In this section we exploit the similarity between workflows and planning domains in order to describe the evolution of a \ournet by means of a planning language. Once the original workflow behaviour has been encoded into an equivalent planning domain, we can use the automatic derivation of plans with specific properties to solve the reachability problem.
In our approach we introduce a new action for each transition (to ease the description we will use the same names) and represent the status of the workflow -- marking and variable assignments -- by means of fluents.
Although their representation as dynamic rules is conceptually similar we will separate the description of the encoding by considering first the behavioural part (the WF-net) and then the encoding of data (variable assignments and guards).

\subsection{Encoding \ournet behaviour}

Since we focus on 1-safe WF-nets the representation of markings is simplified by the fact that each place can either contain 1 token or no tokens at all. This information can be represented introducing a propositional fluent for each place, true iff the corresponding place holds a token.
Let us consider $\tuple{P,T,F}$ the \emph{safe WF-net} component of a \ournet system. The declaration part of the planning domain will include:
\begin{compactitem}
\item a fluent declaration $p$ for each place $p\in P$;
\item an action declaration $t$ for each task $t\in T$.
\end{compactitem}
Since each transition can be fired\footnote{Guards will be introduced in the next section.} only if each input place contains a token, then the corresponding action can be executed when place fluents are true: for each task $t\in T$, given $\{i^t_1, \ldots, i^t_n\} = \pres{t}$, we include the executability condition:
  \begin{lstlisting}
    executable $t$ if $i^t_1, \ldots, i^t_n$.
  \end{lstlisting}
As valid firings are sequential, namely only one transition can be fired at each step, we disable concurrency in the planning domain introducing the following rule for each pair of tasks $t_1,t_2\in T$\footnote{For efficiency reasons we can relax this constraint by disabling concurrency only for transitions sharing places or updating the same variables. This would provide shorter plans.}
  \begin{lstlisting}
    caused false after $t_1$, $t_2$.
  \end{lstlisting}
Transitions transfer tokens from input to output places. Thus the corresponding actions must clear the input places and set the output places to true. This is enforced by including  
  \begin{lstlisting}
    caused -$i^t_1$ after $t$.  $\ldots$ caused -$i^t_n$ after $t$.
    caused $o^t_1$ after $t$.  $\ldots$ caused $o^t_k$ after $t$.
  \end{lstlisting}
for each task $t\in T$ and $\{i^t_1, \ldots, i^t_n\} = \pres{t}\setminus\posts{t}$, $\{o^t_1, \ldots, o^t_k\} = \posts{t}$.
Finally, place fluents should be inertial since they preserve their value unless modified by an action. This is enforced by adding for each $p\in P$
  \begin{lstlisting}
    caused p if not -$p$ after $p$.
  \end{lstlisting}

\vspace{5pt}
\noindent
{\textbf{Planning problem}.} 
Besides the domain described above, a planning problem includes an initial state, and a goal.
In the initial state the only place with a token is the source:
  \begin{lstlisting}
    initially: $start$.
  \end{lstlisting}
The formulation of the goal depends on the actual instance of the reachability problem we need to solve. 
The goal corresponding to the state in which the only place with a token is $end$ is written as:
  \begin{lstlisting}
    goal: $end$, not $p_1$, $\ldots$, not $p_k$?
  \end{lstlisting}
where $\{p_1, \ldots, p_k\} = P\setminus\{ end \}$.

\subsection{Encoding data}

For each variable $v\in \varset$ we introduce a fluent unary predicate $\kvar{v}$ holding the value of that variable. Clearly, $\kvar{v}$ predicates must be functional and have no positive instantiation for undefined variables.

We also introduce auxiliary fluents to facilitate the writing of the rules. Fluent $\kvardef{v}$ indicates whether the $v$ variable is \emph{not} undefined -- it is used both in tests and to enforce models where the variable is assigned/unassigned. The fluent $\kvarchange{v}$ is used to inhibit inertia for the variable $v$ when its value is updated because of the execution of an action.

\ournet includes the specification of the set of values that each transition can write on a variable. This information is static, therefore it is included in the background knowledge by means of a set of unary predicates $\kvardom{v,t}$ as a set of facts:
  \begin{lstlisting}
    $\kvardom{v,t}$(e).
  \end{lstlisting}
  for each $v\in\varset$, $t\in T$, and $e\in \writef(t)(v)$.

\vspace{5pt}
\noindent
{\textbf{Constraints on variables}.} 
For each variable $v\in \varset$:

\begin{compactitem}
  \item we impose functionality
  \begin{lstlisting}
    caused false if $\kvar{v}$(X), $\kvar{v}$(Y), X != Y.
  \end{lstlisting}
  \item we force its value to propagate to the next state unless it is modified by an action ($\kvarchange{v}$)
  \begin{lstlisting}
    caused $\kvar{v}$(X) if not -$\kvar{v}$(X), not $\kvarchange{v}$ 
                         after $\kvar{v}$(X).
  \end{lstlisting}
  \item the defined fluent is the projection of the argument
  \begin{lstlisting}
    caused $\kvardef{v}$ if $\kvar{v}$(X).
  \end{lstlisting}
\end{compactitem}

\vspace{5pt}
\noindent
{\textbf{Variable updates}.} 
The value of a variable is updated by means of causation rules that depend on the transition $t$ that operates on the variable, and depends on the value of $\writef(t)$. For each $v$ in the domain of $\writef(t)$:
\begin{itemize}
  \item $\writef(t)(v) = \emptyset$: delete (undefine) a variable $v$
   \begin{lstlisting}
    caused false if $\kvardef{v}$ after t.
    caused $\kvarchange{v}$ after t.
  \end{lstlisting}
  \item $\writef(t)(v) \subseteq \domf(v)$: set $v$ with a value nondeterministically chosen among a set of elements from its domain
   \begin{lstlisting}
    caused $\kvar{v}$(V) if $\kvardom{v,t}$(V), not -$\kvar{v}$(V) after t.
    caused -$\kvar{v}$(V) if $\kvardom{v,t}$(V), not $\kvar{v}$(V) after t.
    caused false if not $\kvardef{v}$ after t.
    caused $\kvarchange{v}$ after t.
  \end{lstlisting}
   If $\writef(t)(v)$ contains a single element $d$, then the assignment is deterministic and the first three rules above can be substituted with\footnote{The deterministic version is a specific case of the non-deterministic ones and equivalent in the case that there is a single $\kvardom{v,t}(d)$ fact.}
   \begin{lstlisting}
    caused $\kvar{v}$(d) after t.
  \end{lstlisting}
\end{itemize}

\vspace{5pt}
\noindent
{\textbf{Guards}.} 
To each subformula $\varphi$ of transition guards is associated a fluent $\kguard{\varphi}$ that is true when the corresponding formula is satisfied. To simplify the notation, for any transition $t$, we will use $\kguard{t}$ to indicate the fluent $\kguard{\guardf(t)}$.
Executability of transitions is conditioned to the satisfiability of their guards; instead of modifying the executability rule including the $\kguard{t}$ among the preconditions, we use a constraint rule preventing executions of the action whenever its guard is not satisfied:
  \begin{lstlisting}
    caused false after t, not $\kguard{t}$.
  \end{lstlisting}

  Translation of atoms (\katom) is defined in terms of $\kvar{v}$ predicates. For instance $\katom(v = w)$ corresponds to \lstinline|$\kvar{v}$(V)|, \lstinline|$\kvar{w}$(W)|, \lstinline|V == W|. That is $\katom(v,T) = \kvar{t}\text{\lstinline|(T)|}$ for $t\in\V$, and $\katom(d,T) = \kvar{t}\text{\lstinline|T ==\ $d$|}$ for $d\in\bigcup_i \Delta_i$.
For each subformula $\varphi$ of transition guards a static rule is included to ``define'' the fluent $\kguard{\varphi}$:

\begin{tabular}{rl}
$true$ :& {\lstinline|caused $\kguard{\varphi}$ if true .|} \\
$\deff(v)$ :& {\lstinline|caused $\kguard{\varphi}$ if $\kvardef{v}$ .|} \\
$t_1 = t_2$ :& {\lstinline|caused $\kguard{\varphi}$ if $\katom$($t_1$,T1), $\katom$($t_2$,T2), T1 == T2 .|} \\
$t_1 \leq t_2$ :& {\lstinline|caused $\kguard{\varphi}$ if $\katom$($t_1$,T1), $\katom$($t_2$,T2), $\kord$(T1,T2) .|} \\
$\neg \varphi_1$ :& {\lstinline|caused $\kguard{\varphi}$ if not $\kguard{\varphi_1}$ .|} \\
$\varphi_1 \land \ldots \land \varphi_n$ :& {\lstinline|caused $\kguard{\varphi}$ if $\kguard{\varphi_1},\ldots,\kguard{\varphi_n}$ .|}
\end{tabular}
 
\subsection{Correctness and completeness}

We provide a sketch of the correctness and completeness of the encoding. Proofs can be found in~\cite{additional:ICAPS:2017}.

Planning states include all the information to reconstruct the original \ournet states. In fact, we can define a function $\plantopn{\cdot}$ mapping consistent planning states into \ournet states as following: $\plantopn{s} = (\mrk,\assign)$ with

\vspace{-.4cm}
\small
  \begin{align*}
    \forall p\in P,\; \mrk(p) = \begin{cases}
    1 \text{ if } p \in s \\
    0 \text{ otherwise}
  \end{cases} \quad
  \assign = \set{(v,d)\mid \kvar{v}(d)\in s}
  \end{align*}

\noindent  
\normalsize
$\plantopn{s}$ is well defined because $s$ it cannot be the case that $\set{\kvar{v}(d), \kvar{v}(d')}\subseteq s$ with $d\neq d'$, otherwise the static rule
  \begin{lstlisting}
    caused false if $\kvar{v}$(X), $\kvar{v}$(Y), X != Y.
  \end{lstlisting}
    would not be satisfied.
Moreover, 1-safeness implies that we can restrict to markings with range in $\{0,1\}$.
By looking at the static rules we can observe that those defining the predicates $\kvardef{v}$ and $\kguard{t}$ are stratified. Therefore their truth assignment depends only on the extension of $\kvar{v}(\cdot)$ predicates. This implies that $\kguard{t}$ fluents are satisfied iff the variables assignment satisfies the corresponding guard $\guardf(t)$.
Based on these observations, the correctness of the encoding is relatively straightforward since we need to show that a legal transition in the planning domain can be mapped to a valid firing. This is proved by inspecting the dynamic rules.
\begin{lemma}[Correctness]\label{lemm:plantopn:correctness} Let $W$ be a \ournet and $\pdom(W)$ the corresponding planning problem. If $\tuple{s,\{t\},s'}$ is a legal transition in $\pdom(W)$, then $\plantopn{s} \fire{t} \plantopn{s'}$ is a valid~firing~of~$W$.
\end{lemma}

The proof of completeness is more complex because -- given a valid firing -- we need to build a new planning state and show that it is minimal w.r.t.\ the transition. Since the starting state $s$ of $\tuple{s,\{t\},s'}$ does not require minimality we just need to show its existence, while $s'$ must be carefully defined on the basis of the rules in the planning domain.

\begin{lemma}[Completeness]\label{lemm:plantopn:completeness} Let $W$ be a \ournet, $\pdom(W)$ the corresponding planning problem and $(\mrk,\assign)\fire{t}(\mrk',\assign')$ be a valid firing of $W$. Then for each consistent state $s$ s.t.\ $\plantopn{s} = M$ there is a consistent state $s'$ s.t.\ $\plantopn{s'}=M'$ and $\tuple{s,\{t\},s'}$ is a legal transition in $\pdom(W)$.
\end{lemma}


Lemmata \ref{lemm:plantopn:correctness} and \ref{lemm:plantopn:completeness}  provide the basis for the inductive proof of the following theorem:
\begin{theorem}\label{thm:wfnet:eq}
Let $W$ be a \emph{safe WF-net} and $\pdom(PN)$ the corresponding planning problem. Let $(\mrk_0,\assign_0)$ be the initial state of $W$ -- i.e.\ with a single token in the source and no assignments -- and $s_0$ the planning state satisfying the initial condition.
 \begin{description} 
 
	 \small
 \item[$(\Rightarrow)$]
 For \emph{any} case in $W$
 \vspace{-.2cm}
  \begin{displaymath}
    \zeta: (\mrk_0,\assign_0) \fire{t_1} (\mrk_1,\assign_1) \ldots (\mrk_{n-1},\assign_{n-1})\fire{t_n} (\mrk_{n},\assign_{n})
  \end{displaymath}
  there is a trajectory in $\pdom(W)$
  \begin{displaymath}
    \eta: \tuple{s_0,\{t_1\},s_1}, \ldots, \tuple{s_{n-1},\{t_n\},s_{n}}
  \end{displaymath}
  such that $(\mrk_i,\assign_i) = \plantopn{s_i}$ for each $i \in \{0 \ldots n\}$
  and viceversa.
  \item[$(\Leftarrow)$]
  For each trajectory
  \vspace{-.2cm}
  \begin{displaymath}
    \eta: \tuple{s_0,\{t_1\},s_1}, \ldots, \tuple{s_{n-1},\{t_n\},s_{n}}
  \end{displaymath}
   in $\pdom(W)$, the following sequence of firings is a  case of $W$
   \vspace{-.2cm}
  \begin{displaymath}
    \zeta: \plantopn{s_0} \fire{t_1} \plantopn{s_1}\ldots \plantopn{s_{n-1}} \fire{t_n} \plantopn{s_{n}}.
  \end{displaymath}
 \end{description}
\end{theorem}

Theorem~\ref{thm:wfnet:eq} above enables the exploitation of planning techniques to solve the reachability problem in \ournet. Indeed, to verify whether the final marking is reachable it is sufficient to encode it as a condition for the final state and verify the existence of a trajectory terminating in a state where the condition is satisfied.
Decidability of the planning problem is guaranteed by the fact that domains are effectively finite, as in Definition~\ref{def:our-net} the $\writef$ functions range over a finite subset of the domain. 


%% file: 6-relatedworks.tex
\section{Related Work and Conclusions}
\label{sec:related_works}



The key role of data 
 in the context of business processes 
 has been recently recognized.
A number of variants of PNs have been enriched so as to make tokens able to carry data and transitions aware of the data, as in the case 
of Workflow nets enriched with data~\cite{sidorovastahletal:2011,de_leoni:2013}, the model adopted by the business process community. 
In detail, Workflow Net transitions are enriched with information about data (e.g., a variable $request$) and about how it is used by the activity (for reading or writing purposes).
Nevertheless, these nets do not consider data values (e.g., in the example of Section~\ref{sub:trace_repair} we would not be aware of the values of the variable $request$ that \emph{T4} is enabled to write). 
They only allow for the identification of whether the value of the data element is \texttt{defined} or \texttt{undefined}, thus limiting the reasoning capabilities that can be provided on top of them. For instance, in the example of Section~\ref{sub:trace_repair}, we would not be able to discriminate between the worker and the student loan for the trace in (\ref{eq:trace-data}), as we would only be aware that $request$ is \texttt{defined} after \emph{T4}.

The problem of incomplete traces has been investigated in a number of works of trace alignment in the field of process mining, where it still represents one of the challenges. 
Several works
 have addressed the problem of aligning event logs and procedural models, without~\cite{Adriansyahetal:2011} and with~\cite{deLeonietal:2012b,de_leoni:2013} data. 
All these works, however, explore the search space of possible moves in order to find the best one aligning the log and the model.
Differently from them,
 in this work (i) we assume that the model is correct and we focus on the repair of incomplete execution traces; (ii) we want to exploit state-of-the-art planning techniques to reason on control and data flow rather than solving an optimisation problem.   
We can overall divide the approaches facing the problem of reconstructing flows of model activities given a partial set of information in two groups: quantitative and qualitative. 
The former
  rely on the availability of a probabilistic model of execution and knowledge. 
For example, in~\cite{Rogge-Solti2013}, 
the authors exploit stochastic PNs and Bayesian Networks to recover missing information (activities and their durations). 
The latter stand on the idea of describing ``possible outcomes'' regardless of likelihood; hence, knowledge about the world will consist of equally likely ``alternative worlds'' given the available observations in time, as in this work. 
For example, in~\cite{Bertoli2013} the same issue of reconstructing missing information has been tackled by reformulating it in terms of a Satisfiability(SAT) problem rather than as a planning problem.

Planning techniques have already been used in the context of business processes, e.g., for verifying process constraints~\cite{RegisRAM12} 
 or for the construction and adaptation of autonomous process models~\cite{SilvaL11,CoopIS2012}.
In~\cite{DeGiacomo2016} automated planning techniques have been applied for aligning execution traces and declarative models. 
 As in this work, in~\cite{Di-Francescomarino-C.:2015aa}, planning techniques have been used for addressing the problem of incomplete execution traces with respect to procedural models. 
 However, differently from the two approaches above, this work uses for the first time planning techniques to target the problem of completing incomplete execution traces with respect to a procedural model that also takes into account data and the value they can assume. 



%% file: 7-conclusions.tex

Despite this work mainly focuses on the problem of trace completion, the proposed automated planning approach can easily exploit reachability for model satisfiability and trace compliance and furthermore can be easily extended also for aligning data-aware procedural models and execution traces.
 Moreover, the presented encoding in the planning language \klng, can be directly adapted to other action languages with an expressiveness comparable to $\C$~\cite{lif99}.
  In the future, we would like to explore these extensions and implement the proposed approach and its variants in a prototype. 


%% file: additional-backgrounds.tex
\section{Preliminaries} \label{app:preliminaries}

\subsection{Workflow Nets}

\begin{definition}[Petri Net~\cite{de_leoni:2013}]
  A Petri Net is a triple $\tuple{P,T,F}$ where
  \begin{itemize}
    \item $P$ is a set of places;
    \item $T$ is a set of transitions;
    \item $F\subseteq (P \times T) \cup (T \times P)$ is the flow relation describing the ``arcs'' between places and transitions (and between transitions and places).
  \end{itemize}
  
  The \emph{preset} of a transition t is the set of its input places: $\pres{t} = \{p \in P \mid (p,t) \in F\}$. The \emph{postset} of $t$ is the set of its output places: $\posts{t} = \{p \in P \mid (t,p) \in F\}$. Definitions of pre- and postsets of places are analogous.
  
  The \emph{marking} of a Petri net is a total mapping $M : P\mapsto \mathbb{N}$.
\end{definition}

\begin{definition}[WF-net~\cite{sidorovastahletal:2011}]
A Petri net $\tuple{P,T,F}$ is a workflow net (WF-net) if it has a single source place start, a single sink place end, and every place and every transition is on a path from start to end; i.e.\ for all $n\in P\cup T$, $(start,n)\in F^*$ and $(n,end)\in F^*$, where $F^*$ is the reflexive transitive closure of $F$.
\end{definition}

The semantics of a PN is defined in terms of its markings and \emph{valid firing} of transitions which change the marking. A firing of a transition $t\in T$ from $M$ to $M'$ is valid -- denoted by $M \fire{t_0} M$ -- iff:
\begin{itemize}
  \item $t$ is enabled in $M$, i.e., $\{ p\in P\mid M(p)>0\}\supseteq \pres{t}$; and
  \item the marking $M'$ satisfies the property that for every $p\in P$:
  \begin{displaymath}
    M'(p) =
    \begin{cases}
      M(p)-1 & \text{if $p\in \pres{t}\setminus\posts{t}$}\\
      M(p)+1  & \text{if $p\in \posts{t}\setminus\pres{t}$}\\
      M(p) & \text{otherwise}
    \end{cases}
  \end{displaymath}
\end{itemize}
A \emph{case} of PN is a sequence of valid firings
$$M_0 \fire{t_1} M_1, M_1 \fire{t_2} M_2, \ldots, M_{k-1} \fire{t_k} M_k$$ where $M_0$ is the marking where there is a single token in the start place.

\begin{definition}[safeness]
 A marking of a Petri Net is $k$-safe if the number of tokens in all places is at most $k$. A Petri Net is $k$-safe if the initial marking is $k$-safe and the marking of all cases is $k$-safe. 
\end{definition}

In this document we focus on 1-safeness, which is equivalent to the original safeness property as defined in~\cite{vanderaalst:1998}.\footnote{In the following we will use safeness as a synonym of 1-safeness.} Note that for safe nets the range of markings is restricted to $\{0, 1\}$.

\subsection{Action Language \klng}

The formal definition of \klng can be found in Appendix A of~\cite{eiter_dlvk:2003}; here, as reference, we include the main concepts.

We assume disjoint sets of action, fluent and type names, i.e., predicate symbols of arity $\geq 0$, and disjoint sets of constant and variable symbols. Literals can be positive or negative atoms; denoted by $-$. Given a set of literals $L$, $L^+$ (respectively, $L^-$) is the set of positive (respectively, negative) literals in $L$. A set of literals is \emph{consistent} no atoms appear both positive and negated.

The set of all action (respectively, fluent, type) literals is denoted as $\L_{act}$ (respectively, $\L_{fl}$, $\L_{typ}$). 

Furthermore, $\L_{fl,typ} = \L_{fl} \cup \L_{typ}$, $\L_{dyn} = \L_{fl} \cup \L^+_{act}$, and $\L = \L_{fl,typ} \cup \L^+_{act}$.

\begin{definition}[Causation rule]
  A (causation) rule is an expression of the form
    \begin{lstlisting}
caused $f$ if $b_1,\ldots, b_k$, not $b_{k+1}$, $\ldots$, not $b_\ell$ 
         after $a_1,\ldots, a_m$, not $a_{m+1}$, $\ldots$, not $a_n$.
  \end{lstlisting}
were $f\in \L_{fl}\cup \{ false \}$, $b_i\in\L_{fl,typ}$, $a_i\in\L$, $\ell\geq k\geq 0$ and $n\geq m\geq 0$. 

If $n=0$ the rule is called \emph{static}.

We define $h(r) = f$, $pre^+(r) = \{a_1,\ldots, a_m\}$, $pre^-(r) = \{a_{m+1},\ldots, a_n\}$, $post^+(r) = \{b_1,\ldots, b_k\}$, $post^-(r) = \{b_{k+1},\ldots, b_\ell\}$
\end{definition}

\begin{definition}[Initial state constraints]
  An initial state constraint is a static rule preceded by the keyword \lstinline|initially|.
\end{definition}

\begin{definition}[Executability condition]
  An executability condition e is an expression of the form
    \begin{lstlisting}
executable $a$ if $b_1,\ldots, b_k$, not $b_{k+1}$, $\ldots$, not $b_\ell$.
  \end{lstlisting}
were $a\in \L_{act}^+$, $b_i\in\L_{fl,typ}$, and $\ell\geq k\geq 0$.

We define $h(e) = a$, $pre^+(e) = \{b_1,\ldots, b_k\}$, and $pre^-(e) = \{b_{k+1},\ldots, b_\ell\}$
\end{definition}

Since in this document we're dealing with ground plans, for the definition of \emph{typed instantiation} the reader is referred to the original paper.

\begin{definition}[Planning domain, \cite{eiter_dlvk:2003} Def.\ A.5]
An action description $\tuple{D,R}$ consists of a finite set $D$ of action and fluent declarations and a finite set $R$ of safe causation rules, safe initial state constraints, and safe executability conditions. A \klng planning domain is a pair $PD = \tuple{\Pi,AD}$, where $\Pi$ is a stratified Datalog program (the background knowledge) which is safe, and $AD$ is an action description. We call $PD$ positive, if no default negation occurs in AD.
\end{definition}

The set $lit(PD)$ contains all the literals appearing in PD.

\begin{definition}[State, State transition]
A state w.r.t.\ a planning domain PD is any consistent set $s\subseteq\L_{fl} \cap (lit(PD) \cup lit(PD)^-)$ of legal fluent instances and their negations. A state transition is any tuple $t = \tuple{s, A, s'}$ where $s, s'$ are states and $A \subseteq \L_{act} \cap lit(PD)$ is a set of legal action instances in PD.
\end{definition}

Semantics of plans including default negation is defined by means of a Gelfond–Lifschitz type reduction to a positive planning domain.
\begin{definition}
  Let PD be a ground and well-typed planning domain, and let $t = \tuple{s,A,s'}$ be a state transition. Then, the reduction $PD^t$ of PD by $t$ is the planning domain where the set of rules $R$ of PD is substituted by $R^t$ obtained by deleting
  \begin{enumerate}
    \item each $r\in R$,where either $post^-(r) \cap s'\neq\emptyset$ or $pre^-(r)\cap s\neq\emptyset$,and
    \item all default literals \lstinline|not $\ell$| ($\ell\in\L$) from the remaining $r\in R$.
  \end{enumerate}
\end{definition}

\begin{definition}[Legal initial state, executable action set, legal state transition]
For any planning domain $PD = \tuple{D,R}$
\begin{itemize}
  \item a state $s_0$ is a legal initial state, if $s_0$ is the least set s.t.\ for all static and initial rules $r$ $post(r)\subseteq s_0$ implies $h(r) \subseteq s_0$;
  \item a set $A\subseteq\L^+_{act}$ is an executable action set w.r.t.\ a state $s$, if for each $a\in A$ there is an executability condition $e\in R^{\tuple{s,A, \emptyset}}$ s.t.\ $h(e)=\{a\}$, $pre(e) \cap \L_{fl} \subseteq s$, and $pre(e) \cap \L^+_{act} \subseteq A$;
  \item a state transition $t = \tuple{s, A, s'}$ is legal if $A$ is an executable action set w.r.t.\ $s$, and $s'$ is the minimal consistent set that satisfies all causation rules in $R^{\tuple{s,A, s'}}$ w.r.t.\ $s \cup A$. A causation rule $r \in  R^{\tuple{s,A, s'}}$, is satisfied if the three conditions
      \begin{enumerate}
        \item $post(r) \subseteq s'$
        \item $pre(r) \cap \L_{fl} \subseteq s$
        \item $pre(r) \cap \L_{act} \subseteq A$
      \end{enumerate}
      all hold, then $h(r) \neq \{false\}$ and $h(r) \subseteq s'$.
\end{itemize}
\end{definition}

\begin{definition}[Trajectory]
A sequence of state transitions $$\tuple{s_0, A_1, s_1}, \tuple{s_1, A_2, s_2}, \ldots,\tuple{s_{n-1}, A_n, s_n}$$, $n \geq 0$, is a trajectory for PD, if $s_0$ is a legal initial state of PD and all $\tuple{s_{i-1}, A_i, s_i}$, $1\leq i\leq n$, are legal state transitions of PD.

If $n = 0$, then the trajectory is empty.
\end{definition}

\begin{definition}[Planning problem]
  A planning problem is a pair of planning domain PD and a ground goal $q$
    \begin{lstlisting}
$g_1,\ldots, g_m$, not $g_{m+1}$, $\ldots$, not $g_n$.
  \end{lstlisting}
where $g_i\in\L_{ft}$ and $n\geq m\geq 0$.

A state $s$ \emph{satisfies} the goal if $\{g_1,\ldots, g_m\}\subseteq s$ and $\{g_{m+1},\ldots, g_n\}\cap s = \emptyset$.
\end{definition}

\begin{definition}[Optimistic plan]
  A sequence of action sets $A_1,\ldots, A_k$ is an optimistic plan for a planning problem $\tuple{PD, q}$ if there is a trajectory $\tuple{s_0, A_1, s_1}, \ldots,\tuple{s_{k-1}, A_k, s_k}$ establishing the goal $q$, i.e.\ $s_k$ satisfies $q$.
\end{definition}

\begin{definition}[Secure plan]
  An optimistic plan $A_1,\ldots, A_n$ is secure if for every legal initial state $s_0$ and trajectory $\tuple{s_0, A_1, s_1}, \tuple{s_1, A_2, s_2}, \ldots,\tuple{s_{k-1}, A_k, s_k}$ $0\leq k\leq n$, it holds that
  \begin{enumerate}
    \item if $k=n$ then $s_k$ satisfies the goal;
    \item if $k<n$, then there is a legal transition $\tuple{s_k,A_{k+1},s_{k+1}}$.
  \end{enumerate}
\end{definition}

\nocite{vazquez:2014,eiter_dlvk:2003}


%% file: additional-framework.tex
\section{Framework}

\subsection{Data Model}\label{sec:datamodel}

\begin{definition}[Data model]
A data model for is a couple $\D = (\V, \Delta, \domf, \ordf)$ where:
\begin{compactitem}
\item $\V$ is a possibly infinite set of variables;
\item $\Delta = \set{\Delta_1, \ldots, \Delta_n}$ is a set of domains (not necessarily disjoints);
\item $\domf: \V \rightarrow \Delta$ is a total and surjective function which associate to each variable $v$ its finite domain $\Delta_i$;
\item $\ordf$ is a partial function that, given a domain $\Delta_i$, if $\ordf(\Delta_i)$ is defined, then it returns a \emph{partial order} (reflexive, antisymmetric and transitive) $\leq_{\Delta_i} \subseteq \Delta_i \times \Delta_i$.
\end{compactitem}
\end{definition}

\begin{definition}[Assignment]
Let $\D = \tuple{\V, \Delta, \domf, \ordf}$ be a data model. An assignment for variables in $\V$ is a \emph{partial} function $\assign: \V \rightarrow \Delta_v$ such that for each $v \in \V$, if $\assign(v)$ is defined, then we have $\assign(v) \in \domf(v)$. We write $\Phi[\assign]$ for the formula $\Phi$ where each occurrence of a variable $v \in \img(\assign)$ is replaced by $\assign(v)$.
\end{definition}

\begin{definition}[Guard language, syntax]
Given a data model, the language $\guardlang$ of guards is the set of formulas $\Phi$ inductively defined according to the following grammar:
\[ \begin{array}{lcl}
\Phi & := & true \mid \deff(v) \mid t_1 = t_2 \mid t_1 \leq t_2 \mid \neg \Phi_1 \mid \Phi_1 \land \Phi_2
\end{array} \]
where $v \in \V$ and $t_1, t_2 \in \V \cup \bigcup_i \Delta_i$.
\end{definition}

\begin{definition}[Guard language, semantics]
Given a data model $\dmodel$, an assignment $\assign$ and a guard $\Phi \in \guardlang$ we say that $\dmodel, \assign$ satisfies $\Phi$, written $D, \assign \models \Phi$ inductively on the structure of $\Phi$ as follows:
\begin{itemize}
\item $\dmodel, \assign \models true$;
\item $\dmodel, \assign \models \deff(v)$ iff $v \in \img(\assign)$;
\item $\dmodel, \assign \models t_1 = t_2$ iff $t_1[\assign], t_2[\assign] \not \in \V$ and $t_1[\assign] \equiv t_2[\assign]$;
\item $\dmodel, \assign \models t_1 \leq t_2$ iff $t_1[\assign], t_2[\assign] \in \Delta_i$ for some $i$ and $\ordf(\Delta_i)$ is defined and $t_1[\assign] \leq_{\Delta_i} t_2[\assign]$;
\item $\dmodel, \assign \models \neg \Phi$ iff it is not the case that $\dmodel, \assign \models \Phi$;
\item $\dmodel, \assign \models \Phi_1 \land \Phi_2$ iff $\dmodel, \assign \models \Phi_1$ and $\dmodel, \assign \models \Phi_2$.
\end{itemize} 
\end{definition}

\subsection{Petri Nets with Data}\label{sec:our-net}

\begin{definition}[\ournet] 
  A Petri Net with data is a tuple $\tuple{\dmodel, \nmodel, \writef, \guardf}$ where:
\begin{itemize}
    \item $\nmodel=\tuple{P,T,F}$ is a Petri Net;
    \item $\dmodel=\tuple{\V, \Delta, \domf, \ordf}$ is a data model;
    \item $\writef: T \mapsto (\V' \mapsto 2^{\domf(\varset)})$, where $\varset'\subseteq\varset$ and $\writef(t)(v)\subseteq\domf(v)$ for each $v\in \varset'$, is a function that associate each transition to a (\emph{partial}) function mapping variables to a subset of their domain.
    \item $\guardf: T \mapsto \guardlang$ a function that associates a guard expression to each transition.
  \end{itemize}
\end{definition}

The definition of $\writef$ provides a fine grained description of the way that transitions modify the state of the \ournet, enabling the description of different cases:
\begin{itemize}
  \item $\emptyset\subset \writef(t)(v) \subseteq \domf(v)$: transition $t$ nondeterministically assigns a value from $\writef(t)(v)$ to $v$;\footnote{Allowing a subset of  $\domf(v)$ enables the specification of restrictions for specific tasks, e.g., while a task selects among \texttt{yes, no, maybe} another one can only choose between \texttt{yes} and \texttt{no}.}
  \item $\writef(t)(v) = \emptyset$: transition $t$ deletes the value of $v$ (undefined);
  \item $v \not\in dom(\writef(t))$: value of $v$ is not modified by transition $t$.
\end{itemize}

\begin{definition}
A state of a \ournet $\tuple{\dmodel, \nmodel, \writef, \guardf}$ is a pair $(\mrk,\assign)$ where $\mrk$ is a marking for $\tuple{P,T,F}$ and $\assign$ is an assignment. State transitions and firing are adapted to the additional information about data.
\end{definition}

\begin{definition}[Valid Firing]\label{def:dpn:firing}
  Given a \ournet $\tuple{\dmodel, \nmodel, \writef, \guardf}$, a firing of a transition $t\in T$ is valid firing in $(\mrk,\assign)$ resulting in a state $(\mrk',\assign')$ (written as $(\mrk,\assign)\fire{t}(\mrk',\assign')$) iff:
  \begin{itemize}
  \item $t$ is enabled in $\mrk$, i.e., $\{ p\in P\mid \mrk(p)>0\}\supseteq \pres{t}$; and
  \item $\dmodel, \assign \models \guardf(t)$;
  \item the marking $\mrk'$ satisfies the property that for every $p\in P$:
  \begin{displaymath}
    \mrk'(p) =
    \begin{cases}
      \mrk(p)-1 & \text{if $p\in \pres{t}\setminus\posts{t}$}\\
      \mrk (p)+1  & \text{if $p\in \posts{t}\setminus\pres{t}$}\\
      \mrk (p) & \text{otherwise}
    \end{cases}
  \end{displaymath}
  \item the assignment $\assign'$ satisfies the properties that its domain is $$dom(\assign') = dom(\assign)\cup\set{v\mid \writef(t)(v)\neq\emptyset}\setminus\set{v\mid \writef(t)(v)=\emptyset}$$ and for each $v\in dom(\assign')$:
  \begin{displaymath}
    \assign'(v) =
    \begin{cases}
      d \in \writef(t)(v) & \text{if $v\in dom(\writef(t))$}\\
      \assign(v)  & \text{otherwise.}
    \end{cases}
  \end{displaymath}
\end{itemize}
\end{definition}

Cases of \ournet are defined as those of WF-nets, with the only difference that in the initial state the assignment is empty.


%% file: additional-encoding-traces.tex
\section{Trace completion as Reachability}\label{sec:encoding:traces}

Within this document we consider the possibility that some of the activities can be observable or not. In the sense that they \emph{might} or \emph{can never} appear in logs. This enables a fine grained information on the different type of activities that compose a process. For example, is common practice in modelling the introduction of transitions for routing purposes (e.g.\ and-joins) that do not correspond to real activities and as such they would never be observed. On the other end, some activities must be logged by their nature -- e.g.\ a database update -- so if they are not observed we can be sure that they never occurred. We use the term \emph{always observable} for transitions that must appear in the logs and \emph{never observable} for those that would never appear in logs; all other transitions may or may not be present in the logs but they may occur in actual cases.

Since the focus of the paper is on the use of planning techniques to provide reasoning services for workflows with data, we decided to omit this aspect for reasons of space.

\begin{definition}[Trace]
Let Let $P = \tuple{\dmodel, \nmodel, \writef, \guardf}$ be a \ournet.
An event of $P$ is a tuple \tuple{t,w, w^d} where $t\in T$ is a transition, $w \in \domf(\varset)^{\varset'}$ -- with $\varset'\subseteq\varset$ and $w(v) \in \writef(t)(v)$ for all $v\in\varset'$ -- is a partial function that represents the variables written by the execution of $t$, and $w^d\subseteq\varset$ the set of variables deleted (undefined) by the execution of $t$. Obviously, $w^d\cap\varset' = \emptyset$.

A trace of $P$ is a finite sequence of events
$\tau = (e_1,\ldots, e_n)$. In the following we indicate the $i$-th
event of $\tau$ as $\tstep{\tau}{i}$. Given a set of tasks $T$, the set of traces is inductively
defined as follows:
\begin{compactitem}
\item $\epsilon$ is a trace;
\item if $\tau$ is a trace and $e$ an event, then $\tau \cdot e$
  is a trace.
\end{compactitem}
\end{definition}

\begin{definition}[Trace Compliance] \label{def:compliance} A (valid) firing $(\mrk,\assign)\fire{t}(\mrk',\assign')$ is \emph{compliant} with an event \tuple{t',w, w^d} iff $t = t'$, $w^d = \set{v\mid \writef(t')(v)=\emptyset}$, $dom(\assign') = dom(w) \cup dom(\assign) \setminus w^d$, and for all $v\in dom(w)$ $w(v) = \assign'(v)$.

  A case $$(\mrk_0,\assign_0)\fire{t_1}(\mrk_1,\assign_1)\ldots (\mrk_{k-1},\assign_{k-1})\fire{t_k}(\mrk_{k},\assign_{k})$$ is \emph{compliant} with the trace $\tau = (e_1,\ldots, e_\ell)$ iff there is an injective mapping $\gamma$ between $[1\ldots \ell]$ and $[1\ldots k]$ such that:\footnote{If the trace is empty then $\ell=0$ and $\gamma$ is empty.}
  \begin{align}
    \forall i,j \text{ s.t. } 1\leq i < j \leq \ell \quad \gamma(i) < \gamma(j)\\
    \forall i \text{ s.t. } 1\leq i \leq \ell \quad (\mrk_{\gamma(i-1)},\assign_{\gamma(i-1)})\fire{t_{\gamma(i)}}(\mrk_{\gamma(i)},\assign_{\gamma(i)}) \text{ is compliant with } e_i\\
    \forall i \text{ s.t. } 1\leq i \leq k\quad t_i \text{ \emph{always observable} implies } \exists j \text{ s.t. } \gamma(j) = i 
  \end{align}
\end{definition}

We assume that the workflow starts and terminates with special transitions -- indicated by $start_t$ and $end_t$ -- while $start$ and $end$ denote start place and sink respectively.

\begin{definition}[Trace workflow]\label{def:trace:workflow}
Let $W = \tuple{\dmodel, \nmodel = \tuple{P,T,F}, \writef, \guardf}$ be a \ournet and $\tau = (e_1,\ldots, e_n)$ -- where $e_i = \tuple{t_i,w_i, w_i^d}$ -- a trace of $W$. The \emph{trace workflow} $W^\tau = \tuple{\dmodel, \nmodel^\tau = \tuple{P^\tau,T^\tau,F^\tau}, \writef^\tau, \guardf^\tau}$ is defined as following:
\begin{align*}
  P^\tau = {} & P \cup \set{p_{e_0}}\cup\set{p_{e}\mid e\in \tau} & \text{$p_{e_0}$, $p_e$ new places}\\
  T^\tau = {} & T \cup \set{t_e\mid e\in \tau} &\text{$t_e$ new transitions}\\
  F^\tau = {} & F \cup {} \\
  & \set{(t_{e_i},p)\mid i=1\ldots n, (t_i,p)\in F}\cup\set{(p,t_{e_i})\mid i=1\ldots n, (p,t_i)\in F} \cup {}\\
  & \set{(t_{e_i},p_{e_i})\mid i=1\ldots n}\cup\set{(p_{e_{i-1}},t_{e_i})\mid i=1\ldots n} \cup \set{(start_t,p_{e_0}), (p_{e_n},end_t)}\\
  \writef^\tau(t) = {} & 
  \begin{cases}
    \set{(v,\set{j})\mid (v,j)\in w_i}\cup\set{(v,\emptyset)\mid v\in w_i^d} & \text{for $t = t_{e_i}$}\\
    \writef(t) & \text{for $t\in T$}
  \end{cases}
  \\
  \guardf^\tau(t) = {} & 
  \begin{cases}
    \guardf(t_i) & \text{for $t = t_{e_i}$}\\
    \textsf{false} & \text{for $t\in T$ fully observable} \\
    \guardf(t) & \text{for $t\in T$ not fully observable}
  \end{cases}
\end{align*}
\end{definition}

It's not difficult to see that whenever the original \ournet $W$ is a workflow net, then $W^\tau$ is a workflow net as well because the newly introduced nodes are in a the path $start, start_t, p_{e_0}, t_{e_1}, p_{e_1}, \ldots, t_{e_n}, p_{e_n}, end_t, end$.

To relate cases from $W^\tau$ to the original workflow $W$ we introduce a ``projection'' function $\wkfProj_\tau$ that maps elements from cases of the enriched workflow to cases using only elements from the original workflow. To simplify the notation we will use the same name to indicate mappings from states, firings and cases.

\begin{definition}
  Let $W = \tuple{\dmodel, \nmodel = \tuple{P,T,F}, \writef, \guardf}$ be a \ournet, $\tau = (e_1,\ldots, e_n)$ -- where $e_i = \tuple{t_i,w_i, w_i^d}$ a trace of $W$, and $W^\tau = \tuple{\dmodel, \nmodel^\tau = \tuple{P^\tau,T^\tau,F^\tau}, \writef^\tau, \guardf^\tau}$ the corresponding trace workflow. The mapping $\wkfProj_\tau$ is defined as following:
  \begin{enumerate}
    \item let $(\mrk',\assign')$ be a marking of $W^\tau$, then
    $$\wkfProj_\tau(\mrk') = (\mrk'\cap P\times\mathbb{N})$$
    is a state of $W$;
    \item let $(\mrk',\assign')$ be a state of $W^\tau$, then
    $$\wkfProj_\tau((\mrk',\assign')) = (\wkfProj_\tau(\mrk'),\assign')$$
    is a state of $W$;
    \item let $t$ be a transition in $T^\tau$, then
    $$\wkfProj_\tau(t) = \begin{cases}
    t_i & \text{for $t = t_{e_i}$}\\
    t & \text{for $t\in T$}
  \end{cases}$$
  \item let $(\mrk,\assign)\fire{t}(\mrk',\assign')$ be a firing in $W^\tau$, then
  $$\wkfProj_\tau((\mrk,\assign)\fire{t}(\mrk',\assign')) = \wkfProj_\tau((\mrk,\assign))\fire{\wkfProj_\tau(t)}\wkfProj_\tau((\mrk',\assign'))$$
  \item let $C = f_0,\ldots,f_k$ be a case of $W^\tau$, then 
  $$\wkfProj_\tau(C) = \wkfProj_\tau(f_0),\ldots,\wkfProj_\tau(f_k)$$
  \end{enumerate}
\end{definition}

In the following we consider a \ournet $W = \tuple{\dmodel, \nmodel = \tuple{P,T,F}, \writef, \guardf}$ and a trace $\tau = (e_1,\ldots, e_n)$ of $W$ -- where $e_i = \tuple{t_i,w_i, w_i^d}$. Let $W^\tau = \tuple{\dmodel, \nmodel^\tau = \tuple{P^\tau,T^\tau,F^\tau}, \writef^\tau, \guardf^\tau}$ be the corresponding trace workflow. To simplify the notation, in the following we will use $t_{e_0}$ as a synonymous for $start_t$ and $t_{e_{n+1}}$ as $end_t$; as if they were part of the trace.

\begin{lemma}
  Let $C$ be a case of $W^\tau$, then $\wkfProj_\tau(C)$ is a case of $W$.
\end{lemma}
\begin{proof}
  Let $C = (\mrk_0,\assign_0)\fire{t_1}(\mrk_1,\assign_1)\ldots (\mrk_{k-1},\assign_{k-1})\fire{t_k}(\mrk_{k},\assign_{k})$, to show that $\wkfProj_\tau(C)$ is a case of $W$ we need to prove that 
 \begin{inparaenum}[(i)]
  \item $\wkfProj_\tau((\mrk_0,\assign_0))$ is an initial state of $W$ and that 
  \item the firing $\wkfProj_\tau((\mrk_{i-1},\assign_{i-1})\fire{t_i}(\mrk_{i},\assign_{i}))$ is valid w.r.t.\ $W$ for all $1\leq i\leq n$.
\end{inparaenum}
\begin{compactenum}[i)]
\item By definition $\wkfProj_\tau((\mrk_0,\assign_0)) = (\wkfProj_\tau(\mrk_0),\assign')$ and $\wkfProj_\tau(\mrk_0)\subseteq\mrk_0$. Since the start place is in $P$, then start is the only place with a token in $\wkfProj_\tau(\mrk_0)$.
\item Let consider an arbitrary firing $f_i = (\mrk_{i-1},\assign_{i-1})\fire{t_i}(\mrk_{i},\assign_{i})$ in $C$ (valid by definition), then $\wkfProj_\tau(f_i) = (\wkfProj_\tau(\mrk_{i-1}),\assign_{i-1})\fire{\wkfProj_\tau(t_i)}(\wkfProj_\tau(\mrk_{i}),\assign_{i})$.

  Note that -- by construction -- $\guardf(t_i) = \guardf(\wkfProj_\tau(t_i))$, $\posts{\wkfProj_\tau(t_i)} = \posts{t_i}\cap P$, $\pres{\wkfProj_\tau(t_i)} = \pres{t_i}\cap P$, $dom(\writef(t_i)) = dom(\writef(\wkfProj_\tau(t_i)))$ and $\writef(t_i)(v)\subseteq \writef(\wkfProj_\tau(t_i))(v)$ ; therefore
  
    \begin{itemize}
  \item $\set{ p\in P^\tau\mid \mrk_{i-1}>0}\cap P = \{ p\in P\mid \wkfProj_\tau(\mrk_{i-1})>0\}\supseteq \pres{\wkfProj_\tau(t_i)}$ because $\set{ p\in P^\tau\mid \mrk_{i-1}>0}\supseteq \pres{t_i}$;
  \item $\dmodel, \assign \models \guardf(\wkfProj_\tau(t_i))$ because $\dmodel, \assign \models \guardf(t_i)$
  \item for all $p\in P$ $\wkfProj_\tau(\mrk_{j})(p) = \mrk_{j}(p)$, therefore:
  \begin{displaymath}
    \mrk_{i}(p) = \wkfProj_\tau(\mrk_{i})(p) =
    \begin{cases}
      \mrk_{i-1}(p)-1 = \wkfProj_\tau(\mrk_{i-1})(p) - 1 & \text{if $p\in \pres{\wkfProj_\tau(t_i)}\setminus\posts{\wkfProj_\tau(t_i)}$}\\
      \mrk_{i-1}(p)+1 = \wkfProj_\tau(\mrk_{i-1})(p) + 1 & \text{if $p\in \posts{\wkfProj_\tau(t_i)}\setminus\pres{\wkfProj_\tau(t_i)}$}\\
      \mrk_{i-1}(p) = \wkfProj_\tau(\mrk_{i-1})(p) & \text{otherwise}
    \end{cases}
  \end{displaymath}
  because $f_i$ is valid w.r.t.\ $W^\tau$;
  \item the assignment $\assign_i$ satisfies the properties that its domain is $$dom(\assign_i) = dom(\assign_{i-1})\cup\set{v\mid \writef(\wkfProj_\tau(t_i))(v)\neq\emptyset}\setminus\set{v\mid \writef(\wkfProj_\tau(t_i))(v)=\emptyset}$$ and for each $v\in dom(\assign_i)$:
  \begin{displaymath}
    \assign_i(v) =
    \begin{cases}
      d \in \writef(t_i)(v)\subseteq \writef(\wkfProj_\tau(t_i))(v) & \text{if $v\in dom(\writef(t_i)) = dom(\writef(\wkfProj_\tau(t_i)))$}\\
      \assign_{i-1}(v)  & \text{otherwise.}
    \end{cases}
  \end{displaymath}
  because $f_i$ is valid.
\end{itemize}
\end{compactenum}

\end{proof}

Before going into details, we will consider some properties of the ``trace'' workflow.

\begin{lemma}\label{lemma:trace:places}
  Let $W = \tuple{\dmodel, \nmodel = \tuple{P,T,F}, \writef, \guardf}$ be a \ournet and $\tau = (e_1,\ldots, e_n)$ -- where $e_i = \tuple{t_i,w_i, w_i^d}$ -- a trace of $W$. If $C = (\mrk_0,\assign_0)\fire{t_1}(\mrk_1,\assign_1)\ldots (\mrk_{k-1},\assign_{k-1})\fire{t_k}(\mrk_{k},\assign_{k})$ is a case of $W^\tau$ then for all $0\leq i\leq k$: $$\Sigma_{p\in P^\tau\setminus P} \mrk_i(p)\leq \mrk_0(start)$$
\end{lemma}
\begin{proof}
  By induction on the length of $C$.
  \begin{itemize}
    \item For $k=1$ then the only executable transition is $start_t$, therefore $t_1 = start_t$ which -- by assumption -- has two output places and -- by construction -- $\posts{start_t}\setminus P = \set{p_{e_0}}$. Since the firing is valid, then $\mrk_1(p_{e_0}) = \mrk_0(p_{e_0}) + 1 = 1\leq \mrk_0(start)$.
    \item Let's assume that the property is true a case $C$ of length $n$ and consider $C' = C (\mrk_{n},\assign_{n})\fire{t_{n+1}}(\mrk_{n+1},\assign_{n+1})$. By construction, each $p\in P^\tau\setminus P$ has a single incoming edge and $\set{t\in T^\tau\mid e_i\in \posts{t}} = \set{t_{e_i}}$ and $\set{t\in T^\tau\mid e_i\in \pres{t}} = \set{t_{e_{i+1}}}$. Therefore the only occurrence in which a $p_{e_i}\in P^\tau\setminus P$ can \emph{increase} its value is when $t_{n+1} = t_{e_i}$. Since the transition is valid, then $\mrk_{n+1}(p_{e_i}) = \mrk_{n}(p_{e_i}) + 1$ and $\mrk_{n+1}(p_{e_{i-1}}) = \mrk_{n}(p_{e_{i-1}}) - 1$; therefore $\Sigma_{p\in P^\tau\setminus P} \mrk_i(p) = \Sigma_{p\in P^\tau\setminus P} \mrk_{i-1}(p)\leq \mrk_0(start)$ -- by the inductive hypothesis.
  \end{itemize}
\end{proof}

\begin{lemma}\label{lemma:trace:places:seq}
  Let $W = \tuple{\dmodel, \nmodel = \tuple{P,T,F}, \writef, \guardf}$ be a \ournet and $\tau = (e_1,\ldots, e_n)$ -- where $e_i = \tuple{t_i,w_i, w_i^d}$ -- a trace of $W$, $C = (\mrk_0,\assign_0)\fire{t_1}(\mrk_1,\assign_1)\ldots (\mrk_{k-1},\assign_{k-1})\fire{t_k}(\mrk_{k},\assign_{k})$ a case of $W^\tau$, and $t_{e_i}$ is a transition of a firing $f_m$ in $C$ with $1\leq i\leq n$, then
    \begin{inparaenum}[(i)]
    \item $t_{e_{i-1}}$ is in a transition of a firing in $C$ that precedes $f_m$, 
    \item and if $\mrk_0(start)=1$ then there is a single occurrence of $t_{e_i}$ in $C$.
    \end{inparaenum}

\end{lemma}
\begin{proof}
  The proof for the first part follows from the structure of the workflow net; because -- by construction -- each $p\in P^\tau\setminus P$ has a single incoming edge and $\set{t\in T^\tau\mid e_i\in \posts{t}} = \set{t_{e_i}}$ and $\set{t\in T^\tau\mid e_i\in \pres{t}} = \set{t_{e_{i+1}}}$.
  Since each firing must be valid -- if $f_m = (\mrk_{m-1},\assign_{m-1})\fire{t_{e_i}}(\mrk_{m},\assign_{m})$ is in $C$, then $\mrk_{m-1}(p_{e_{i-1}})\geq 1$ and this can only be true if there is a firing $f_r = (\mrk_{r-1},\assign_{r-1})\fire{t_{e_{i-1}}}(\mrk_{r},\assign_{r})$ in $C$ s.t.\ $r<m$.
  
  To prove the second part is enough to show that for each $1\leq i\leq n$, if $t_{e_i}$ appears more than once in $C$ then there must be multiple occurrences of $t_{e_{i-1}}$ as well. In fact, if this is the fact, then we can use the previous part to show that there must be multiple occurrences of $t_{e_0} = start$, and this is only possible if $\mrk_0(start)>1$. 
  
  By contradiction let's assume that there are two firings $f_m$ and $f_m'$, with $m<m'$, with the same transition $t_{e_i}$, but there is only a single occurrence of $t_{e_{i-1}}$ in a firing $f_r$. Using the previous part of this lemma we conclude that $r<m<m'$, therefore $\mrk_{m-1}(p_{e_{i-1}}) = 1$ because a token could be transferred into $p_{e_{i-1}}$ only by $t_{e_{i-1}}$, so $\mrk_{m}(p_{e_{i-1}}) = 0$. In the firings between $m$ and $m'$ there are no occurrences of $t_{e_{i-1}}$, so $\mrk_{m'-1}(p_{e_{i-1}}) = \mrk_{m}(p_{e_{i-1}}) = 0$ which is in contradiction with the assumption that $f_m'$ is a valid firing. 
\end{proof}

Now we're ready to show that the ``trace'' workflow characterises all and only the cases compliant wrt\ the given trace. We divide the proof into correctness and completeness.

\begin{lemma}[Correctness]\label{lemma:wkf:trace:encoding:crct} Let $C = (\mrk_0,\assign_0)\fire{t_1}(\mrk_1,\assign_1)\ldots (\mrk_{k-1},\assign_{k-1})\fire{t_k}(\mrk_{k},\assign_{k})$ be a case of $W^\tau$ s.t.\ $\mrk_0(start) = 1$, and $\ell = max(\set{i\mid t_i\text{ is in a firing of }C}\cup\set{0})$, then the case $\wkfProj_\tau(C)$ of $W$ is compliant with $\tau' = (e_1,\ldots, e_\ell)$ or the empty trace if $\ell$ is $0$.
\end{lemma}

\begin{proof}
  By induction on the length of $C$.
  \begin{itemize}
    \item If $C = (\mrk_0,\assign_0)\fire{t_1}(\mrk_1,\assign_1)$ then $t_1 = start_t$ because the firing is valid and the only place with a token in $\mrk_0$ is $start$; therefore $\ell = 0$ and $\tau'$ is the empty trace. $C$ trivially satisfy the empty trace because no observable transitions are in $\wkfProj_\tau(C)$.
    \item Let $C = (\mrk_0,\assign_0)\fire{t_1}(\mrk_1,\assign_1)\ldots (\mrk_{k-1},\assign_{k-1})\fire{t_k}(\mrk_{k},\assign_{k})$ s.t.\ $\wkfProj_\tau(C)$ is compliant with $\tau'$. Let's consider $C' = C \cdot (\mrk_{k},\assign_{k})\fire{t_{k+1}}(\mrk_{k+1},\assign_{k+1})$: either $t_{k+1}\in T^\tau\setminus T$ or $t_{k+1}\in T$.
    In the first case $t_{k+1} = t_{e_\ell}$ for some $1\leq \ell\leq n$, and -- by using Lemma~\ref{lemma:trace:places:seq} -- in $C$ there are occurrences of all the $t_{e_i}$ for $1\leq i<\ell$ and it's the only occurrence of $t_{e_\ell}$. This means that $\ell = max(\set{i\mid t_i\text{ is in a firing of }C}\cup\set{0})$ and we can extend $\gamma$ to $\gamma'$ by adding the mapping from $\ell$ to $k+1$. The mapping is well defined because of the single occurrence of $t_{e_\ell}$. By definition of $t_{e_\ell}$, $(\mrk_{k},\assign_{k})\fire{t_{k+1}}(\mrk_{k+1},\assign_{k+1})$ is compliant with ${e_\ell}$ and the mapping $\wkfProj_\tau$ preserve the assignments, therefore $\wkfProj_\tau(\mrk_{k},\assign_{k})\fire{t_{k+1}}(\mrk_{k+1},\assign_{k+1})$ is compliant with ${e_\ell}$ as well. By using the inductive hypnotises we can show that $C'$ is compliant as well.
    In the second case the mapping is not modified, therefore the inductive hypothesis can be used to provide evidence of the first two conditions for trace compliance of Definition~\ref{def:compliance}. For the third (transitions always observable) it's sufficient to consider that $t_{k+1}$ cannot be always observable because its guard is never satisfiable in $W^\tau$.
  \end{itemize}
\end{proof}

\begin{lemma}[Completeness]\label{lemma:wkf:trace:encoding:cmpl} Let $C = (\mrk_0,\assign_0)\fire{t_1}(\mrk_1,\assign_1)\ldots (\mrk_{k-1},\assign_{k-1})\fire{t_k}(\mrk_{k},\assign_{k})$ be a case of $W$ compatible with $\tau = (e_1,\ldots, e_n)$, then there is a case $C'$ of $W^\tau$ s.t.\ $\wkfProj_\tau(C') = C$.
\end{lemma}
\begin{proof}
  Since $C$ is compliant with $\tau$, then there is a mapping $\gamma$ satisfying the conditions of Definition~\ref{def:compliance}. Let $C' = (\mrk'_0,\assign_0)\fire{t'_1}(\mrk'_1,\assign_1)\ldots (\mrk'_{k-1},\assign_{k-1})\fire{t'_k}(\mrk'_{k},\assign_{k})$ a sequence of firing of $W^\tau$ defined as following:
  \begin{itemize}
    \item $\mrk'_0 = \mrk_0\cup\set{(p_{e_i},0)\mid 0\leq i\leq n}$
    \item $t'_1 = t_1$ and $\mrk'_1 = \mrk_1\cup\set{(p_{e_j},0)\mid 1\leq j\leq n} \cup\set{(p_{e_0},1)}$
    \item for each $(\mrk'_{i-1},\assign_{i-1})\fire{t'_i}(\mrk'_{i},\assign_{i})$, $2\leq i\leq n$:
    \begin{itemize}
    \item if there is $\ell$ s.t.\ $\gamma(\ell) = i$ then $t'_i = t_{e_\ell}$ and $$\mrk'_i = \mrk_i\cup\set{(p_{e_j},0)\mid 0\leq j\leq n, j\neq \ell} \cup\set{(p_{e_\ell},1)}$$
    \item otherwise $t'_i = t_i$ and $$\mrk'_i = \mrk_i\cup(\mrk'_{i-1}\cap (P^\tau\setminus P)\times\mathbb{N})$$
    \end{itemize}
  \end{itemize}
  It's not difficult to realise that by construction $\wkfProj_\tau(C') = C$. 
  
  To conclude the proof we need to show that $C'$ is a case of $W^\tau$. Clearly $(\mrk'_0,\assign_0)$ is a starting state, so we need to show that all the firings are valid. The conditions involving variables -- guards and update of the assignment -- follows from the fact that the original firings are valid and the newly introduced transitions are restricted according to the trace data.
  
  Conditions on input and output places that are both in $W$ and $W^\tau$ are satisfied because of the validity of the original firing. The newly introduced places satisfy the conditions because of the compliance wrt the trace, which guarantees that for each firing with transition $t_{e_\ell}$ there is the preceding firing with transition $t_{e_{\ell-1}}$ that put a token in the $p_{e_{\ell-1}}$ place.
\end{proof}

\begin{theorem}
  Let $W$ be a \ournet and $\tau = (e_1,\ldots, e_n)$ a trace; then $W^\tau$ characterises all and only the cases of $W$ compatible with $\tau$. That is
  \begin{compactitem}
    \item[$\Rightarrow$] if $C$ is a case of $W^\tau$ containing $t_{e_n}$ then $\wkfProj_\tau(C)$ is compatible with $\tau$; and 
    \item[$\Leftarrow$] if $C$ is a case of $W$ compatible with $\tau$, then there is a case $C'$ of $W^\tau$ s.t.\ $\wkfProj_\tau(C') = C$.
  \end{compactitem}
\end{theorem}
\begin{proof}\ 

\begin{itemize}
  \item[$\Rightarrow$] If $C$ is a case of $W^\tau$ containing $t_{e_n}$, then $\ell$ of Lemma~\ref{lemma:wkf:trace:encoding:crct} is $n$ therefore $\tau' = \tau$ and $\wkfProj_\tau(C)$ is compatible with $\tau$.
  \item[$\Leftarrow$] If $C$ is compatible with $\tau$ then by Lemma~\ref{lemma:wkf:trace:encoding:cmpl} there is a case $C'$ of $W^\tau$ s.t.\ $\wkfProj_\tau(C') = C$. 
\end{itemize}
 
\end{proof}

\begin{theorem}\label{thm:safeness:trace-wrkf}
  Let $W = \tuple{\dmodel, \nmodel = \tuple{P,T,F}, \writef, \guardf}$ be a \ournet and $\tau = (e_1,\ldots, e_n)$ -- where $e_i = \tuple{t_i,w_i, w_i^d}$ -- a trace of $W$. If $W$ is $k$-safe then $W^\tau$ is $k$-safe as well.
\end{theorem}

\begin{proof}
  We prove the theorem by induction on the length of a case $C = (\mrk_0,\assign_0)\fire{t_1}(\mrk_1,\assign_1)\ldots (\mrk_{k-1},\assign_{k-1})\fire{t_k}(\mrk_{k},\assign_{k})$. Note that by construction, for any marking $M'$ of $W^\tau$ and $p\in P$, $M'(p) = \wkfProj_\tau(M')(p)$.
  \begin{itemize}
    \item For a case of length 1 the property trivially holds because by definition $\mrk_0(start) \leq k$ and for each $p\in P^\tau$ (different from $start$) $\mrk_0(start) = 0$, and since $(\mrk_0,\assign_0)\fire{t_1}(\mrk_1,\assign_1)$ is valid the only case in which the number of tokens in a place is increased is for $p\in \posts{t_1}\setminus\pres{t_1}$. For any $p$ different from $start$ this becomes $1\leq k$; while since the $start$ place -- by assumption -- doesn't have any incoming arc therefore $\mrk_1(start) = \mrk_0(start) - 1\leq k$.
    \item For the inductive step we assume that each marking $\mrk_0,\ldots \mrk_{m-1}$ is $k$-safe. By contradiction we assume that $\mrk_{m}$ is not $k$-safe; therefore there is a place $p\in P^\tau$ s.t.\ $\mrk_{m}>k$. There are two cases, either $p\in P^\tau\setminus P$ or $p\in P$. In the first case there is a contradiction because, by Lemma~\ref{lemma:trace:places}, $\Sigma_{p\in P^\tau\setminus P} \mrk_i(p)\leq \mrk_0(start) = k$. In the second case, since $\wkfProj_\tau(C)$ is a case of $W$ and $\wkfProj_\tau(\mrk_{m})(p) = \mrk_{m}(p)$, there is a contradiction with the hypothesis that $W$ is $k$-safe.
  \end{itemize}
\end{proof}

%% file: additional-encoding.tex
\section{Encoding Reachability as Planning Problem} \label{app:encoding}

\subsection{Encoding WF-nets behaviour}

Let $PN = \tuple{\dmodel, \nmodel=\tuple{P,T,F}, \writef, \guardf}$ be a safe \ournet be a \emph{safe WF-net}, we define the \emph{planning problem} $\pdom(W) = \tuple{\Pi, D, R, q}$ by introducing a fluent for each place and an action for each task. Execution and causation rules constraint the plan to mimic the behaviour of the petri net.

\paragraph{Declarations}
\begin{itemize}
\item $D$ contains a fluent declaration $p$ for each place $p\in P$;
\item $D$ contains an action declaration $t$ for each task $t\in T$;
\end{itemize}

\paragraph{Executability rules}
\begin{itemize}
\item actions are executable if each input place has a token; i.e.\
  for each task $t\in T$, given $\{i^t_1, \ldots, i^t_n\} = \pres{t}$,
  there's an executability rule:
  \begin{lstlisting}
    executable $t$ if $i^t_1, \ldots, i^t_n$.
  \end{lstlisting}
\end{itemize}

\paragraph{Causation rules}
\begin{itemize}
\item parallelism is disabled; for each pair of tasks $t_1,t_2\in T$
  there's the rule:\footnote{There's a \klng macro to disable
    concurrency. In practice concurrency could be enabled for actions
    that do not share input or output places.}
  \begin{lstlisting}
    caused false after $t_1$, $t_2$.
  \end{lstlisting}
\item after the execution of a task, input conditions must be
  ``cleared'' and tokens moved to the output ones; for each task
  $t\in T$ and
  $\{i^t_1, \ldots, i^t_n\} = \pres{t}\setminus\posts{t}$,
  $\{o^t_1, \ldots, o^t_k\} = \posts{t}$:
  \begin{lstlisting}
    caused -$i^t_1$ after $t$.  $\ldots$ caused -$i^t_n$ after $t$.
    caused $o^t_1$ after $t$.  $\ldots$ caused $o^t_k$ after $t$.
  \end{lstlisting}
\item the positive state of the places is inertial (i.e.\ must be explicitly
  modified); for each $p\in P$:
  \begin{lstlisting}
    caused p if not -$p$ after $p$.
  \end{lstlisting}
\end{itemize}

\paragraph{Initial state}
\begin{itemize}
\item The only place with a token is the source:
  \begin{lstlisting}
    initially: i.
  \end{lstlisting}
\end{itemize}

\paragraph{Goal}{\ }\\

The formulation of the goal depends on the actual instance of the reachability problem we need to solve. E.g.\ it can be a specific marking:
\begin{itemize}
\item The only place with a token is the sink:
  \begin{lstlisting}
    goal: o, not $p_1$, $\ldots$, not $p_k$?
  \end{lstlisting}
  where $\{p_1, \ldots, p_k\} = P\setminus\{o\}$.
\end{itemize}

\subsection{Encoding of Data}

To each variable $v\in \varset$ corresponds to a \emph{inertial} fluent predicate $\kvar{v}$ with a single argument ``holding'' the value of the variable, and a ``domain'' predicate $\kvardom{v}$ representing the domain of the variable. Unset variables have no positive instantiation of the $\kvar{v}$ predicate. The predicate $\kvar{v}$ must be functional.

We introduce also auxiliary fluents that indicate whether a variable is \emph{not} undefined $\kvardef{v}$ -- used both in tests and to enforce models where the variable is assigned/unassigned -- and $\kvarchange{v}$ to ``inhibit'' inertia when variables might change because of the result of an action.

\paragraph{Constraints on variables}

For each variable $v\in \varset$:

\begin{itemize}
  \item functionality
  \begin{lstlisting}
    caused false if $\kvar{v}$(X), $\kvar{v}$(Y), X != Y.
  \end{lstlisting}
  \item variable defined predicate
  \begin{lstlisting}
    caused $\kvardef{v}$ if $\kvar{v}$(X).
  \end{lstlisting}
  \item variable fluents are inertial unless they can be modified by actions
  \begin{lstlisting}
    caused $\kvar{v}$(X) if not -$\kvar{v}$(X), not $\kvarchange{v}$ after $\kvar{v}$(X).
  \end{lstlisting}
  \item the background knowledge ($\Pi$) includes the set of facts:
  \begin{lstlisting}
    $\kvardom{v,t}$(d).
  \end{lstlisting}
  for each $v\in\varset$, $t\in T$, and $d \in \writef(t)(v)$.
\end{itemize}

\paragraph{Guards} To each task $t$ is associated a fluent $\kguard{t}$ that is true when the corresponding guard is satisfied. Instead of modifying the executability rule including the $\kguard{t}$ among the preconditions, we use a constraint rule ruling out executions of the action whenever its guard is not satisfied:
  \begin{lstlisting}
    caused false after t, not $\kguard{t}$.
  \end{lstlisting}
  This equivalent formulation simplify the proofs because of its incremental nature (there are just additional rules).

  Translation of atoms (\katom) is defined in terms of $\kvar{v}$ predicates, e.g., $\katom(v = w)$ corresponds to \lstinline|$\kvar{v}$(V), $\kvar{w}$(W), V == W|. The \lstinline|$\kvardef{v}$| predicate can be used to test whether a variable is defined, or undefined, i.e.\ \lstinline|not $\kvardef{v}$|.

  The guard $\guardf(t) = (a_{1,1}\land\ldots\land a_{1,n_1})\lor\ldots\lor(a_{k,1}\land\ldots\land a_{k,n_k})$ where each $a_{i,j}$ is an atom, corresponds to the set of rules for $\kguard{t}$:\footnote{Arbitrary expressions can be easily translated by introducing new fluents for the subexpressions.}
   \begin{lstlisting}
    caused $\kguard{t}$ if $\katom(a_{1,1})$, $\ldots$, $\katom(a_{1,n_1})$.
    $\vdots$
    caused $\kguard{t}$ if $\katom(a_{k,1})$, $\ldots$, $\katom(a_{1,n_k})$.
  \end{lstlisting}
 
\paragraph{Variables update} The value of a variable is updated by means of causation rules that depend on the task $t$ that operates on the variable:
\begin{itemize}
  \item $\writef(t)(v) = \emptyset$: delete (undefine) a variable $v$
   \begin{lstlisting}
    caused false if $\kvardef{v}$ after t.
    caused $\kvarchange{v}$ after t.
  \end{lstlisting}
  \item $\writef(t)(v) \subseteq \domf(v)$: set $v$ with a value nondeterministically chosen among a set of elements from its domain
   \begin{lstlisting}
    caused $\kvar{v}$(V) if $\kvardom{v,t}$(V), not -$\kvar{v}$(V) after t.
    caused -$\kvar{v}$(V) if $\kvardom{v,t}$(V), not $\kvar{v}$(V) after t.
    caused $\kvarchange{v}$ after t.
    caused false if not $\kvardef{v}$ after t.
  \end{lstlisting}
   If $\writef(t)(v)$ contains a single element $e$, then there the assignment is deterministic and the above rules can be substituted with\footnote{The deterministic version is a specific case of the non-deterministic ones and equivalent in the case that there is a single $\kvardom{v,t}(d)$ fact. In the following, the proofs will consider the general non-deterministic formulation only.}
   \begin{lstlisting}
    caused $\kvar{v}$(d) after t.
    caused $\kvarchange{v}$ after t.
  \end{lstlisting}
\end{itemize}

\paragraph{Guards}
To each subformula $\varphi$ of transition guards is associated a fluent $\kguard{\varphi}$ that is true when the corresponding formula is satisfied. To simplify the notation, for any transition $t$, we will use $\kguard{t}$ to indicate the fluent $\kguard{\guardf(t)}$.

Executability of transitions is conditioned to the satisfiability of their guards:
  \begin{lstlisting}
    caused false after t, not $\kguard{t}$.
  \end{lstlisting}

  Translation of atoms (\katom) is defined in terms of $\kvar{v}$ predicates. We assume a binary $\kord$ predicate representing the partial order among the elements of the domains. We also assume that elements of $\bigcup_i \Delta_i$ can be directly represented by constants of \klng language.

  For $t \in \V \cup \bigcup_i \Delta_i$ and $T$ a \klng variable we define
  \begin{displaymath}
  \katom(t,T) = 
    \begin{cases}
      \kvar{t}\text{\lstinline|(T)|} & \text{for $t\in\V$}\\
      \text{\lstinline|T ==\ $t$|} & \text{for $t\in\bigcup_i \Delta_i$}
    \end{cases}
  \end{displaymath}

For each subformula $\varphi$ of transition guards a static rule is included to ``define'' the fluent $\kguard{\varphi}$:
\begin{tabular}{rl}
$true$ :& {\lstinline|caused $\kguard{\varphi}$ if true .|} \\
$\deff(v)$ :& {\lstinline|caused $\kguard{\varphi}$ if $\kvardef{v}$ .|} \\
$t_1 = t_2$ :& {\lstinline|caused $\kguard{\varphi}$ if $\katom$($t_1$,T1), $\katom$($t_2$,T2), T1 == T2 .|} \\
$t_1 \leq t_2$ :& {\lstinline|caused $\kguard{\varphi}$ if $\katom$($t_1$,T1), $\katom$($t_2$,T2), $\kord$(T1,T2) .|} \\
$\neg \varphi_1$ :& {\lstinline|caused $\kguard{\varphi}$ if not $\kguard{\varphi_1}$ .|} \\
$\varphi_1 \land \ldots \land \varphi_n$ :& {\lstinline|caused $\kguard{\varphi}$ if $\kguard{\varphi_1},\ldots,\kguard{\varphi_n}$ .|}
\end{tabular}

\subsection{Correctness and completeness}

\begin{definition}[\plantopn{\cdot} function]
  Let $W = \tuple{\dmodel, \nmodel=\tuple{P,T,F}, \writef, \guardf}$ be a safe \ournet, $\M$ the set of its markings, $\H$ the set of all assignments,
  $\pdom(W)$ the corresponding planning problem and $\S$
  the set of its states, namely, the set of all consistent set of
  ground fluent literals. We define the function
  $\plantopn{\cdot}: \S \rightarrow \M\times\H$ mapping planning and \ournet states. For any consistent $s\in\S$, $\plantopn{s} = (\mrk,\assign)$ is defined as follows:
  \begin{align*}
    \forall p\in P\; \mrk(p) = \begin{cases}
    1 \text{ if } p \in s \\
    0 \text{ otherwise}
  \end{cases}\\
  \assign = \set{(v,d)\mid \kvar{v}(d)\in s}
  \end{align*}
\end{definition}

The function $\plantopn{\cdot}$ is well defined because $s$ is assumed to be consistent therefore it cannot be the case that $\set{\kvar{v}(d), \kvar{v}(d')}\subseteq s$ with $d\neq d'$ otherwise the static rule
  \begin{lstlisting}
    caused false if $\kvar{v}$(X), $\kvar{v}$(Y), X != Y.
  \end{lstlisting}
    would not be satisfied.

Moreover, since we assume that $W$ is safe, we can restrict $\M$ to markings with range restricted to $\{0,1\}$ and there is not loss of information between markings and planing states.

The function $\plantopn{\cdot}$ is not injective because of the \emph{strongly negated atoms}. However it can be shown that if two states differ on the \emph{positive} atoms then the corresponding \ournet states are different as well:

\begin{lemma}\label{lemma:plantopn:inject}
  Let $s$ and $s'$ consistent states in $S$, then $s\cap L^+ \neq s'\cap L^+$ implies $\plantopn{s}\neq\plantopn{s'}$.
\end{lemma}

Observing the static rules (those without the \lstinline|after| part) it can be noted those defining the predicates $\kvardef{v}$ and $\kguard{t}$ are stratified, therefore their truth assignment depends only on the extension of $\kvar{v}(\cdot)$ predicates. This fact can be used to show that

\begin{lemma}[Guards translation]\label{lemma:plantopn:guards}
  Let $s\in S$ satisfying the static rules of $\pdom(W)$, and $\varphi$ a subformula of transition guards in $W$. Given $\plantopn{s} = (\mrk,\assign)$, $\kguard{\varphi}\in s$ iff $\dmodel, \assign \models \varphi$.
\end{lemma}
\begin{proof}
  We prove the lemma by structural induction on $\varphi$. First we consider the base cases.
  \begin{description}
    \item[$true$ :] trivially satisfied because \lstinline{true} is in consistent state.
    \item[$\deff(v)$ :] the only rule where $\kvardef{v}$ is in the head is
  \begin{lstlisting}
    caused $\kvardef{v}$ if $\kvar{v}$(X).
  \end{lstlisting}
  therefore $\kvardef{v}\in s$ iff there is a constant $d$ s.t.\ $\kvar{v}(d)\in s$, and that is the case iff $v\in dom(\assign)$.
  \item[$t_1 = t_2$ :] for the sake of simplicity we consider only the case in which $t_1\equiv v$ is a variable and $t_2\equiv d$ is a constant; the other 3 combinations can be demonstrated in the same way. With this assumption, the only rule with $\kguard{\varphi}$ in the head is
  \begin{lstlisting}
    caused $\kguard{\varphi}$ if $\kvar{v}$(T1), T2 == $d$, T1 == T2 .
  \end{lstlisting}
  therefore $\kguard{\varphi}\in s$ iff $\kvar{v}(d)\in s$, and this is the case iff $\assign(v) = d$.
  \item[$t_1 \leq t_2$ :] this case is analogous to the previous one, where we consider the predicate \lstinline|$\kord$(T1,T2)| instead of equality. Since $\kord$ facts correspond to the orders defined in $\dmodel$, then we can conclude.
  \end{description}
  
  For the inductive step we assume that the property holds for subformulae $\varphi_1,\varphi_2$.
  \begin{description}
    \item[$\neg \varphi_1$ :] the only rule with $\kguard{\varphi}$ in the head is
  \begin{lstlisting}
    caused $\kguard{\varphi}$ if not $\kguard{\varphi_1}$ .
  \end{lstlisting}
  therefore $\kguard{\varphi}\in s$ iff $\kguard{\varphi_1}\not\in s$. We can use the inductive hypothesis to conclude that this is the case iff $\dmodel, \assign \not\models \varphi_1$, that is $\dmodel, \assign \models \varphi$.
  \item[$\varphi_1 \land \ldots \land \varphi_n$ :] the only rule with $\kguard{\varphi}$ in the head is
  \begin{lstlisting}
    caused $\kguard{\varphi}$ if $\kguard{\varphi_1},\ldots,\kguard{\varphi_n}$ .
  \end{lstlisting}
  therefore $\kguard{\varphi}\in s$ iff $\set{\kguard{\varphi_1},\ldots,\kguard{\varphi_n}}\subseteq s$. We can use the inductive hypothesis to show that this is the case iff $\dmodel, \assign \models \varphi_1\land \ldots\varphi_n$ because they are all ground terms.
  \end{description}
\end{proof}

Looking at the guard translation rules and the proof of Lemma~\ref{lemma:plantopn:guards} it is not difficult to realise that according to the structure of the guards some of the rules are redundant and can be simplified. E.g.\ $\kvardef{v}$ can be used in place of $\kguard{\kvardef{v}}$, \lstinline|not $\kguard{\varphi}$| in place of $\kguard{\kvardef{\neg\varphi}}$, and $t_1 = t_2$ can be expanded in place of $\kguard{t_1 = t_2}$ unless they are in the scope of a negation.

\begin{lemma}[Completeness]\label{lemm:plantopn:completeness} Let $W$ be a \emph{safe \ournet} and $\pdom(W)$ the corresponding planning problem.

Let $(\mrk,\assign)\fire{t}(\mrk',\assign')$ be a valid firing of $W$, then for each consistent state $s$ s.t.\ $\plantopn{s} = M$ there is a consistent state $s'$ s.t.\ $\plantopn{s'}=M'$ and $\tuple{s,\{t\},s'}$ is a legal transition in $\pdom(W)$.
\end{lemma}

\begin{proof}
 
  Let $s$ be a consistent state s.t.\ $\plantopn{s} = M$. Note that such $s$ exists because $\plantopn{\cdot}$ involves only the positive literals; therefore any consistent set $s'$ s.t.\ $\{ p \in P\mid M(p)>0\}\cup\set{\kvar{v}(d)\mid (v,d)\in\assign} \subseteq s'$ and $s'\cap (\{ p \in P\mid M(p)<1\}\cup\varset\times\domf(\varset)\setminus\assign$ satisfies the property that $\plantopn{s'} = M$.
  
  We define a new state $s'$ such that $\tuple{s,\{t\},s'}$ is a legal state transition and such that $\plantopn{s'} = M'$; this new state is the union of the following parts:
  \begin{align*}
    s'_{P^+} &= \set{p \in P\mid M'(p)>0} & s'_{P^-} &= \set{-p \mid p\in \pres{t} \setminus \posts{t}} \\
    s'_{\varset^+} &= \set{\kvar{v}(d)\mid (v,d)\in\assign'} & s'_{\varset^-} &= \set{-\kvar{v}(d)\mid d\in\writef(t)(v), (v,d)\not\in\assign'}\\
    s'_{\varset^\downarrow} &= \set{\kvardef{v}\mid \kvar{v}(d)\in s'_{\varset^+}}\\
    s'_{\varset^c} &= \set{\kvarchange{v}\mid v\in dom(\writef(t))}\\
    s'_{\writef} &= \set{\kvardom{v,t}(d)\mid \forall v,t,d . d\in \writef(t)(v)}\\
    s'_{\guardf} &= \set{\kguard{t}\mid \forall t . \mrk,\assign\models\guardf(t) }
 \end{align*}

  By construction $\plantopn{s'} = M$ and it is consistent: $s'_{P^+}\cap s'_{P^-}=\emptyset$ because the fact that $(\mrk,\assign)\fire{t}(\mrk',\assign')$ is a valid firing implies $p\in \pres{t}\setminus\posts{t}$ $M'(p) = 0$, and $s'_{\varset^+}\cap s'_{\varset^-}=\emptyset$ because their conditions are mutually exclusive.
    
  Since $(\mrk,\assign)\fire{t}(\mrk',\assign')$ is valid, then $\pres{t}\subseteq s$ because $\plantopn{s} = M$, therefore the corresponding executable condition with $t$ in the head
  \begin{lstlisting}
    executable $t$ if $i^t_1, \ldots, i^t_n$.
  \end{lstlisting}
  is satisfied.
  
  We need to show that all the causation rules in $\pdom(W)^{\tuple{s,\{t\}, s'}}$ are satisfied and that $s'$ is minimal.
  
  \begin{itemize}
  \item For each pair of tasks $t_1,t_2$, the positive rule:
    \begin{lstlisting}
      caused false after $t_1$, $t_2$.
    \end{lstlisting}
    is satisfied because there is only a task $t$ in the action set.
  \item Consider the rules 
  \begin{lstlisting}
      caused -$i^a_1$ after $a$.  $\ldots$ caused -$i^a_n$ after $a$.
      caused $o^a_1$ after $a$.  $\ldots$ caused $o^a_k$ after $a$.
    \end{lstlisting}
    where $\{i^a_1, \ldots, i^a_n\} = \pres{a}\setminus\posts{a}$,
    $\{o^a_1, \ldots, o^a_k\} = \posts{a}\setminus\pres{a}$. For all $a\neq t$ they are satisfied because the \lstinline|after| condition is false. For $a = t$ the validity of $(\mrk,\assign)\fire{t}(\mrk',\assign')$ ensures that $\pres{t}\setminus\posts{t}\subseteq s'_{P^-}$ and $\posts{t}\subseteq s'_{P^+}$, therefore the rules are satisfied.
    \item For each $p\in P$:
    \begin{lstlisting}
      caused $p$ if not -$p$ after $p$.
    \end{lstlisting}
    we consider the three cases where $p\in\pres{t} \setminus \posts{t}$, $p\in\posts{t}$, or $p\not\in(\pres{t}\cup\posts{t})$.
    \begin{description}
  \item[$p\in\pres{t} \setminus \posts{t}$] then $-p\in s'_{P^-}$ by construction, therefore the rule is not in $\pdom(W)^{\tuple{s,\{t\}, s'}}$
      
  \item[$p\in\posts{t}$] then $M'(p) = 1$ and by construction $p\in s'_{P^+}$ and $-p\not\in s'_{P^+}$ because $s'$ is consistent, so the rule \lstinline|caused $p$ after $p$.| is in $\pdom(W)^{\tuple{s,\{t\}, s'}}$. This rule is satisfied if $p\in s$ and also if $p\not\in s$.
 
  \item[$p\not\in(\pres{t}\cup\posts{t})$] then $M'(p) = M(p)$. If $p\in s'_{P^+}$ the  rule \lstinline|caused $p$ after $p$.| is in $\pdom(W)^{\tuple{s,\{t\}, s'}}$ it's satisfied regardless of the value of $M(p)$; on the other end, if $p\not\in s'$ then $M(p) = 0$ therefore even if \lstinline|caused $p$ after $p$.| would be in $\pdom(W)^{\tuple{s,\{t\}, s'}}$ then it'd be satisfied because its after part is false.
  \end{description}
  \item Functionality rules
  \begin{lstlisting}
    caused false if $\kvar{v}$(X), $\kvar{v}$(Y), X != Y.
  \end{lstlisting}
  is satisfied by construction of $s'_{\varset^+}$
  \item Variable defined predicate rules
  \begin{lstlisting}
    caused $\kvardef{v}$ if $\kvar{v}$(X).
  \end{lstlisting}
  are satisfied by construction of $s'_{\varset^\downarrow}$.
  \item variable fluents are inertial
 \item The background knowledge facts
  \begin{lstlisting}
    $\kvardom{v,t}$(d).
  \end{lstlisting}
  are satisfied by construction of $s'_{\writef}$
  \item The guard predicates rules are satisfied by Lemma~\ref{lemma:plantopn:guards} and the construction of $s'_{\guardf}$.
  \end{itemize}
  
  For rules involving the $\kvar{v}$ predicates (including intertiality rules) we consider the three cases: $v \not\in dom(\writef(t))$, $\writef(t)(v) = \emptyset$, and $\writef(t)(v)\neq\emptyset$. Note that, since the transition includes only $t$, all the rules in $\pdom(W)^{\tuple{s,\{t\}, s'}}$ with a different action in the \lstinline|after| part are satisfied; therefore we focus on the remaining ones.
  
  \begin{description}
  \item[$v \not\in dom(\writef(t))$:] in this case the only rule in $\pdom(W)^{\tuple{s,\{t\}, s'}}$ to verify is the inertial one 
  \begin{lstlisting}
    caused $\kvar{v}$(X) if not -$\kvar{v}$(X), not $\kvarchange{v}$ after $\kvar{v}$(X).
  \end{lstlisting}
  and by construction $-\kvar{v}(d)\not\in s'_{\varset^-}$ for any $d$ and $\kvarchange{v}\not\in s'_{\varset^c}$. This would be not satisfied only in the case that for some $d$ $\kvar{v}(d)\in s$ and $\kvar{v}(d)\not\in s$ -- which means that $(v,d)\in\assign$ and $(v,d)\not\in\assign$ -- but his would be in contradiction with the fact that $(\mrk,\assign)\fire{t}(\mrk',\assign')$ is a valid firing.
  
  \item[$\writef(t)(v) = \emptyset$:] in this case the corresponding rules are
  \begin{lstlisting}
    caused false if $\kvardef{v}$ after t.
    caused $\kvarchange{v}$ after t.
    caused $\kvar{v}$(X) if not -$\kvar{v}$(X), not $\kvarchange{v}$ after $\kvar{v}$(X).
  \end{lstlisting}
  Since there is no $d$ s.t.\ $(v,d)\in\assign'$ then $\kvar{v}(d') \not\in s'_{\varset^+}$ for any $d'$, therefore $\kvardef{v}\not\in s'_{\varset^\downarrow}$ and the first rule is satisfied. The second rule is satisfied by construction of $s'_{\varset^c}$, and the third is not be in $\pdom(W)^{\tuple{s,\{t\}, s'}}$ because $\kvarchange{v}\in s'_{\varset^c}$.

  \item[$\writef(t)(v)\neq\emptyset$:] the rules are
  \begin{lstlisting}
    caused $\kvar{v}$(V) if $\kvardom{v,t}$(V), not -$\kvar{v}$(V) after t.
    caused -$\kvar{v}$(V) if $\kvardom{v,t}$(V), not $\kvar{v}$(V) after t.
    caused $\kvarchange{v}$ after t.
    caused false if not $\kvardef{v}$ after t.
    caused $\kvar{v}$(X) if not -$\kvar{v}$(X), not $\kvarchange{v}$ after $\kvar{v}$(X).
  \end{lstlisting}
  The first two rules are satisfied by construction of $s'_{\varset^+}$ and $s'_{\varset^-}$, while the third by $s'_{\varset^c}$. The fourth because of the fact that the firing is valid, therefore there is a value $d\in\writef(t)(v)$ s.t.\ $(v,d)\in\assign'$, so $\kvar{v}(d)\in s'_{\varset^+}$ and $\kvardef{v}\in s'_{\varset^\downarrow}$. Last rule is not in $\pdom(W)^{\tuple{s,\{t\}, s'}}$ because $\kvarchange{v}\in s'_{\varset^c}$.
  \end{description}
  
To demonstrate the minimality of $s'$ we need to show that removing one literal from any of the components $s'_{P^+}, s'_{P^-}, s'_{\varset^+}, s'_{\varset^-}, s'_{\varset^\downarrow}, s'_{\varset^c}, s'_{\writef}, s'_{\guardf}$ results in some of the rules not being satisfied.

\begin{description}
\item[$s'_{P^+}$] any $p\in s'_{P^+}$ is either in $\posts{t}$ or not. In the first case removing it would not satisfy the rule
  \begin{lstlisting}
    caused $p$ after $t$.
  \end{lstlisting}
  while in the second it would not satisfy the inertial rule
    \begin{lstlisting}
    caused p if not -$p$ after $p$.
  \end{lstlisting}
  because $-p\not\in s'_{P^-}$ and $p\in s$ since the firing is valid.
\item[$s'_{P^-}$] removing $-p$ from $s'_{P^-}$ would not satisfy the rule
  \begin{lstlisting}
    caused -$p$ after $t$.
  \end{lstlisting}

\item[$s'_{\writef}$] removing $\kvardom{v,t}(d)$ from $s'_{\writef}$ would not satisfy the rule
  \begin{lstlisting}
    $\kvardom{v,t}$(d).
  \end{lstlisting}

\item[$s'_{\varset^+}$] let be $\kvar{v}(d)\in s'_{\varset^+}$: either $v\in dom(\writef(t))$ or not. In the first case the rule
   \begin{lstlisting}
    caused $\kvar{v}$(d) if $\kvardom{v,t}$(d), not -$\kvar{v}$(d) after t.
  \end{lstlisting}
  would not be satisfied because $-\kvar{v}(d)\not\in s'_{\varset^-}$ since by assumption $(v,d)\in\assign'$. In the second case the inertial rule
  \begin{lstlisting}
    caused $\kvar{v}$(d) if not -$\kvar{v}$(d), not $\kvarchange{v}$ after $\kvar{v}$(d).
  \end{lstlisting}
  would not be satisfied because $-\kvar{v}(d)\not\in s'_{\varset^-}$, $\kvarchange{v}\not\in s'_{\varset^c}$, and $\kvar{v}(d)\in s$ since the firing is valid.

\item[$s'_{\varset^-}$] removing $-\kvar{v}(d)$ from $s'_{\varset^-}$ would not satisfy rule
   \begin{lstlisting}
    caused -$\kvar{v}$(d) if $\kvardom{v,t}$(d), not $\kvar{v}$(d) after t.
  \end{lstlisting}
  because $\kvar{v}(d)\not\in s'_{\varset^+}$ since $(v,d)\not\in\assign'$.
\item[$s'_{\varset^\downarrow}$] removing any of the $\kvardef{v'}\in \set{\kvardef{v}\mid \kvar{v}(d)\in s'_{\varset^+}}$ would contradict one of the rules
  \begin{lstlisting}
    caused $\kvardef{v'}$ if $\kvar{v'}$(d).
  \end{lstlisting}
  since there is a an element $d'$ s.t.\ $\kvar{v'}(d')\in s'_{\varset^+}$
\item[$s'_{\varset^c}$] removing any $\kvarchange{v'}\in\set{\kvarchange{v}\mid v\in dom(\writef(t))}$
  since $v'\in dom(\writef(t))$, so therefore there is the rule
  \begin{lstlisting}
    caused $\kvarchange{v'}$ after t.
  \end{lstlisting}
  that would not be satisfied.
\item[$s'_{\guardf}$] removing $\kguard{t}$ from $s'_{\guardf}$ would contradict one of the guard rules according to Lemma~\ref{lemma:plantopn:guards}.
\end{description}
\end{proof}

\begin{lemma}[Correctness]\label{lemm:plantopn:correctness} Let $W$ be a \emph{safe \ournet} and $\pdom(W)$ the corresponding planning problem.

If $\tuple{s,\{t\},s'}$ is a legal transition in $\pdom(W)$, then $\plantopn{s} \fire{t} \plantopn{s'}$ is a valid firing of $W$.
\end{lemma}

\begin{proof}
  Let $(\mrk,\assign) = \plantopn{s}$ and $(\mrk',\assign') = \plantopn{s'}$; to show that $\plantopn{s} \fire{t} \plantopn{s'}$ is a valid firing of $W$ (see Definition~\ref{def:dpn:firing}) we need to show that:
    \begin{compactenum}
  \item $t$ is enabled in $\mrk$, i.e., $\{ p\in P\mid \mrk(p)>0\}\supseteq \pres{t}$; and
  \item $\dmodel, \assign \models \guardf(t)$;
  \item the marking $\mrk'$ satisfies the property that for every $p\in P$:
  \begin{displaymath}
    \mrk'(p) =
    \begin{cases}
      \mrk(p)-1 & \text{if $p\in \pres{t}\setminus\posts{t}$}\\
      \mrk (p)+1  & \text{if $p\in \posts{t}\setminus\pres{t}$}\\
      \mrk (p) & \text{otherwise}
    \end{cases}
  \end{displaymath}
  \item the assignment $\assign'$ satisfies the properties that its domain is $$dom(\assign') = dom(\assign)\cup\set{v\mid \writef(t)(v)\neq\emptyset}\setminus\set{v\mid \writef(t)(v)=\emptyset}$$ and for each $v\in dom(\assign')$:
  \begin{displaymath}
    \assign'(v) =
    \begin{cases}
      d \in \writef(t)(v) & \text{if $v\in dom(\writef(t))$}\\
      \assign(v)  & \text{otherwise.}
    \end{cases}
  \end{displaymath}
\end{compactenum}

Since $\tuple{s,\{t\},s'}$ is a legal transition, then the action $t$ must be executable, therefore the rule:
  \begin{lstlisting}
    executable $t$ if $i^t_1, \ldots, i^t_n$.
  \end{lstlisting}
  with $\{i^t_1, \ldots, i^t_n\} = \pres{t}$ must be satisfied in $s$, that is $\pres{t}\subseteq s$ and $\mrk(i^t_j)=1$ for $1\leq j\leq n$.
  
Since $\tuple{s,\{t\},s'}$ is a legal transition, then the rule:
  \begin{lstlisting}
    caused false after t, not $\kguard{t}$.
  \end{lstlisting}
must be satisfied, therefore its body should be false. This means that $\kguard{t}\in s$ and by using Lemma~\ref{lemma:plantopn:guards} we can conclude that $\dmodel, \assign \models \guardf(t)$.

To verify the condition on $\mrk'$, for each $p\in P$ we consider the three cases:
  \begin{description}
  \item[$p\in\pres{t}\setminus\posts{t}$] then in $\pdom(W)$ there is the rule
    \begin{lstlisting}
      caused -$p$ after $t$.
    \end{lstlisting}
    therefore $p\not\in s'$ and $\plantopn{s'}(p)=0$

  \item[$p\in\posts{t}\setminus\pres{t}$] then in $\pdom(W)$ there is the rule
    \begin{lstlisting}
      caused $p$ after $t$.
    \end{lstlisting}
    therefore $p\in s'$ and $\plantopn{s'}(p)=1$
      
  \item[$p\not\in (\pres{t}\setminus\posts{t})\cup(\posts{t}\setminus\pres{t})$] in this case none of the bodies of rules with $p$ (or $-p$) in the head and an action in the body are satisfied because the only executed action is $t$. Therefore the only ``active'' rule having $p$ (or $-p$) in the head can be the ``inertial'' one for the positive atom:
    \begin{lstlisting}
      caused $p$ if not -$p$ after $p$.
    \end{lstlisting}
    Since rules with $-p$ in the head have their bodies falsified $-p\not\in s'$. This means that the rule \lstinline|caused $p$ after $p$.| is in $\pdom(W)^{\tuple{s,\{t\}, s'}}$.
    
    If $\plantopn{s'}(p)=0$ then $p\not\in s'$ therefore $p\not \in s$ otherwise the inertial rule would not be satisfied; so $\plantopn{s}(p)=0$.
    
    If $\plantopn{s'}(p)=1$ and $\plantopn{s}(p)=0$, then $s'$ would not be minimal because $s'\setminus\{p\}$ satisfies the only ``active'' rule with $p$ in the head, therefore $\plantopn{s}(p)=1$.
  \end{description}
  
  Now we verify the conditions on $\assign'$ and for each $v\in\varset$ we consider three distinct cases: $v \not\in dom(\writef(t))$, $\writef(t)(v) = \emptyset$, and $\writef(t)(v)\neq\emptyset$. First we should note that $\kvarchange{v}\in s'$ iff $v\in dom(\writef(t))$, therefore only in the two latter cases where the inertial rule
  \begin{lstlisting}
    caused $\kvar{v}$(X) if not -$\kvar{v}$(X), not $\kvarchange{v}$ after $\kvar{v}$(X).
  \end{lstlisting}
  would not be in $\pdom(W)^{\tuple{s,\{t\}, s'}}$.
  
  \begin{description}
  \item[$v \not\in dom(\writef(t))$:] In this case, the only \emph{active} rule where $\kvar{v}(\cdot)$ appears in the head is the inertial
  \begin{lstlisting}
    caused $\kvar{v}$(X) if not -$\kvar{v}$(X), not $\kvarchange{v}$ after $\kvar{v}$(X).
  \end{lstlisting}
  while there are no rules with $-\kvar{v}(\cdot)$ in the head, because for all actions $t'\neq t$ are ``false'' in $s$. Therefore $-\kvar{v}(d)\not\in s$ and $\kvarchange{v}\not\in s'$ so $\kvar{v}(d)\in s'$ iff $\kvar{v}(d)\in s$. This means that $v\in dom(\assign')$ iff $v\in dom(\assign)$, and $v\in dom(\assign)$ implies that $\assign'(v)=\assign(v)$.
  
  \item[$\writef(t)(v) = \emptyset$:] in this case if $\kvar{v}(d)\in s'$ for some $d$, then $\kvardef{v}\in s'$ as well; therefore the rule
   \begin{lstlisting}
    caused false if $\kvardef{v}$ after t.
    caused $\kvarchange{v}$ after t.
  \end{lstlisting}
  would not be satisfied contradicting the hypothesis that $\tuple{s,\{t\},s'}$ is a legal transition. 
  \item[$\writef(t)(v)\neq\emptyset$:] in this case $\pdom(W)$ contains the rules
  \begin{lstlisting}
    caused $\kvar{v}$(V) if $\kvardom{v,t}$(V), not -$\kvar{v}$(V) after t.
    caused -$\kvar{v}$(V) if $\kvardom{v,t}$(V), not $\kvar{v}$(V) after t.
    caused false if not $\kvardef{v}$ after t.
    caused $\kvarchange{v}$ after t.
  \end{lstlisting}
  Since $\kvardef{v}\in s'$ otherwise the third rule would not be satisfied, there there must be a $d$ s.t.\ $\kvar{v}(d)\in s'$, and this means that $v\in dom(\assign')$. Let assume that $d \not\in \writef(t)(v)$, then it means that $\kvardom{v,t}(d)\not\in s'$ therefore none of the rules with $\kvar{v}(d)$ in the head would be satisfied in $\pdom(W)^{\tuple{s,\{t\},s'}}$ that contradicts the minimality of $s'$. 
  \end{description}
  The analysis of the three cases confirms that the fourth condition is satisfied as well.
\end{proof}

\begin{theorem}\label{thm:wfnet:eq}
Let $W$ be a \emph{safe WF-net} and $\pdom(W)$ the corresponding planning problem. Let $(\mrk_0,\assign_0)$ be the initial state of $W$ -- i.e.\ with a single token in the source and no assignments -- and $s_0$ the planning state satisfying the initial condition.
 \begin{description} 
 \item[$(\Rightarrow)$]
 For \emph{any} case
  \begin{displaymath}
    \zeta: (\mrk_0,\assign_0) \fire{t_1} (\mrk_1,\assign_1) \ldots (\mrk_{n-1},\assign_{n-1})\fire{t_n} (\mrk_{n},\assign_{n})
  \end{displaymath}
  in $W$ there is a trajectory in $\pdom(W)$
  \begin{displaymath}
    \eta: \tuple{s_0,\{t_1\},s_1}, \ldots, \tuple{s_{n-1},\{t_n\},s_{n}}
  \end{displaymath}
  such that $(\mrk_i,\assign_i) = \plantopn{s_i}$ for each $i \in \{0 \ldots n\}$
  and vice versa.
  \item[$(\Leftarrow)$]
  For each trajectory
  \begin{displaymath}
    \eta: \tuple{s_0,\{t_1\},s_1}, \ldots, \tuple{s_{n-1},\{t_n\},s_{n}}
  \end{displaymath}
   in $\pdom(W)$ the sequence of firings 
  \begin{displaymath}
    \zeta: \plantopn{s_0} \fire{t_1} \plantopn{s_1}\ldots \plantopn{s_{n-1}} \fire{t_n} \plantopn{s_{n}}
  \end{displaymath}
    is a case of $W$.
 \end{description}
\end{theorem}

\begin{proof}
  We first prove the left-to-right direction by induction on the length of the case.
  \begin{compactitem}
  \item Base case: by construction, $\plantopn{s_0}= (\mrk_0,\assign_0)$ because of the structure of the initial state.
  \item Inductive case: we consider a case of size $n+1$. By inductive hypothesis, for the case $(\mrk_0,\assign_0) \fire{t_1} (\mrk_1,\assign_1) \ldots (\mrk_{n-1},\assign_{n-1})\fire{t_n} (\mrk_{n},\assign_{n})$ there is a trajectory $\tuple{s_0,\{t_1\},s_1}, \ldots, \tuple{s_{n-1},\{t_n\},s_{n}}$ s.t.\ $\plantopn{s_i}=M_i$ for each $i \in \{0 \ldots n \}$.
  
   Since $s_n$ is consistent and $\plantopn{s_{n}}=(\mrk_{n},\assign_{n})$, by Lemma~\ref{lemm:plantopn:completeness}, there is a state $s_{n+1}$ s.t.\ $\tuple{s_{n}, \{t_n\}, s_{n+1}}$ is a legal transition and $\plantopn{s_{n+1}}=(\mrk_{n},\assign_{n})$ thus proving the claim.
  \end{compactitem}

\medskip

  The right-to-left direction can be proved -- in the same way as the other case -- by induction on the length trajectories by using the Lemma~\ref{lemm:plantopn:correctness}.
\end{proof}
